\documentclass[onefignum,onetabnum]{siamonline190516}

\usepackage{import}
\usepackage{xifthen}
\usepackage{pdfpages}
\usepackage{transparent}

\usepackage{tikz-cd}

\usepackage{crossreftools}
\usepackage{enumitem}
\setlist[enumerate]{leftmargin=.5in}
\setlist[itemize]{leftmargin=.5in}

\newsiamremark{remark}{Remark}
\newsiamremark{hypothesis}{Hypothesis}
\crefname{hypothesis}{Hypothesis}{Hypotheses}
\newsiamthm{claim}{Claim}

\headers{What Kinds of Functions do DNNs Learn?}{R. Parhi and R. D. Nowak}

\title{What Kinds of Functions do Deep Neural Networks Learn? \\ Insights from Variational Spline Theory\thanks{\vspace{-1em}%
\funding{This research was partially supported by ONR MURI grant N00014-20-1-2787, AFOSR/AFRL grant FA9550-18-1-0166, and the NSF  Graduate  Research  Fellowship  Program  under  grant  DGE-1747503.}}}

\author{Rahul Parhi\thanks{Department of Electrical and Computer Engineering,
University of Wisconsin--Madison, Madison, WI 53706
(\email{rahul@ece.wisc.edu}, \email{rdnowak@wisc.edu}).}
\and Robert D. Nowak\footnotemark[1]}

\author{Rahul Parhi\thanks{Department of Electrical and Computer Engineering,
University of Wisconsin--Madison, Madison, WI 53706
(\email{rahul@ece.wisc.edu}, \email{rdnowak@wisc.edu}).}
\and Robert D. Nowak\footnotemark[2]}

\usepackage{macros}

\begin{document}

\maketitle

\begin{abstract}
    We develop a variational framework to understand the properties of functions learned by fitting deep neural networks with rectified linear unit activations to data. We propose a new function space, which is reminiscent of classical bounded variation-type spaces, that captures the compositional structure associated with deep neural networks. We derive a representer theorem showing that deep ReLU networks are solutions to regularized data fitting problems over functions from this space.  The function space consists of compositions of functions from the  Banach spaces of second-order bounded variation in the Radon domain. These are Banach spaces with sparsity-promoting norms, giving insight into the role of sparsity in deep neural networks. The neural network solutions have skip connections and rank bounded weight matrices, providing new theoretical support for these common architectural choices. The variational problem we study can be recast as a finite-dimensional neural network training problem with regularization schemes related to the notions of weight decay and path-norm regularization.  Finally, our analysis builds on techniques from variational spline theory, providing new connections between deep neural networks and splines.
\end{abstract}



\section{Introduction}

A fundamental problem in signal processing, machine learning, and statistics is to estimate an unknown function from possibly noisy measurements. Specifically, in supervised learning, the goal is to find a mapping $f: \R^d \to \R^D$ that agrees (in some sense) with a scattered data set $\curly{(\vec{x}_n, \vec{y}_n)}_{n=1}^N \subset \R^d \times \R^D$, i.e., $\vec{y}_n \approx f(\vec{x}_n)$, $n = 1, \ldots, N$. As there are infinitely many functions that can agree with any given data set, this problem is inherently ill-posed. To circumvent this, some form of \emph{regularization} is imposed on the learning problem. This problem was classically solved via kernel methods, which are solutions to regularized variational problems over reproducing kernel Hilbert spaces (RKHS)~\cite{theory-reproducing-kernels, spline-models-observational}. While these variational problems are infinite-dimensional, the RKHS representer theorem~\cite{spline-rep-thm,generalized-rep-thm} says there exists a unique, parametric, solution to problem, allowing the problem to be recast as a finite-dimensional optimization. Kernel methods (even before the term ``kernel methods'' was coined) have had widespread success dating all the way back to the 1960s, especially due to the tight connections between kernels, reproducing kernel Hilbert spaces, and splines~\cite{splines-minimum, scattered-data, spline-models-observational}.

However, the last decade has shown that deep neural networks often outperform kernel methods in a wide variety of tasks, ranging from  speech recognition~\cite{speech-recognition} to image classification~\cite{imagenet} to solving inverse problems in imaging~\cite{deep-inverse-imaging}.  Thus, there is great interest in understanding the properties of functions learned from data by neural networks, particularly with the rectified linear unit (ReLU) activation function, which is widey used in practice~\cite{deep-learning}.
%
%
The work of~\cite{min-norm-nn-splines,ridge-splines} has proven \emph{Banach space representer theorems} for single-hidden layer neural networks with ReLU activations by considering variational problems over certain Banach spaces. In the univariate case, this space is the classical Banach spaces of second-order bounded variation functions, and the neural network solutions are exactly the well-known locally adaptive linear splines~\cite{fisher-jerome,locally-adaptive-regression-splines,L-splines}. In the multivariate case, this space is the Banach spaces of second-order bounded variation functions in the \emph{Radon domain}.  It is shown in~\cite{min-norm-nn-splines,ridge-splines} that training sufficiently wide finite-width neural networks via gradient descent with weight decay~\cite{weight-decay} leads to solutions of these variational problems. Due to the similarities of the variational problems studied in~\cite{min-norm-nn-splines,ridge-splines} with those studied in variational spline theory, we refer to the neural networks in the multivariate case as \emph{ridge splines} of degree one since single-hidden layer neural networks are simply superpositions of ridge functions and the functions are multivariate continuous piecewise linear functions.

This paper extends this characterization to deep (multi-layer) neural networks with ReLU activation functions. We also remark that a special property of deep ReLU networks is that their input-output relation is continuous piecewise-linear (CPwL)~\cite{number-linear-regions-relu}. The reverse is also true in that any CPwL function can be represented with a sufficiently wide and deep ReLU network~\cite{understanding-deep-relu}. Thus, one can interpret a deep ReLU network as a multivariate spline of degree one. This connection between deep neural networks and splines has been observed by a number of authors~\cite{heirarchical-splines, balestriero2018spline, balestriero2020mad, min-norm-nn-splines, ridge-splines, representer-deep, deep-splines-lipschitz, deep-splines}. In particular, one can view a deep neural network as a hierarchical or deep spline~\cite{heirarchical-splines,balestriero2018spline,balestriero2020mad,representer-deep, deep-splines-lipschitz, deep-splines} to emphasize the compositional nature of deep neural networks. Due to this special property, we will work exclusively with ReLU activation functions in this paper, though all of our results are straightforward to extend to any truncated power activation function.

\subsection{Contributions}
This paper develops a new variational framework to understand the properties of functions learned by deep neural networks fit to data. In particular, we derive a \emph{representer theorem} for the standard fully-connected feedforward deep ReLU network architecture. We show that there exist solutions to a certain variational problem that are realizable by a deep ReLU network. Moreover, these deep ReLU networks have skip connections rank bounded weight matrices. The number of hidden layers and the rank bounds of the weight matrices are hyperparameters to the variational problem and are therefore controllable a priori.  We refer to the neural network solutions as \emph{deep ridge splines} of degree one due to the similarity of the variational problem studied in this paper with the variational problems studied in variational spline theory. This paper contributes the following new results:

\begin{enumerate}
    \item We propose a new function space, which is reminiscent of classical bounded variation-type spaces, that captures the compositional structure associated with deep neural networks by considering functions that are compositions of functions from the Banach spaces studied in our previous work~\cite{ridge-splines}.
    \item We prove a representer theorem that shows that deep ReLU networks with skip connections and rank bounded weight matrices are solutions to regularized data-fitting problems over functions in the compositional Banach spaces.
    \item The regularizer in the variational problem corresponds to the sum of the Banach norms of each function in the composition. These are sparsity-promoting norms. Moreover, these regularizers can be expressed in terms of neural network parameters, suggesting several new, principled forms of regularization for deep ReLU networks that promote sparse (in the sense of the number of active neurons) solutions. These regularizers are related to the notion of ``weight decay'' in neural network training as well as path-norm regularization. 
\end{enumerate}

\subsection{Connections to empirical studies in deep learning}
Our results provide new theoretical support and insight for a number of empirical findings in deep learning.  We show that the common neural network regularization method of  ``weight decay"  \cite{neyshabur2015search} corresponds to Radon domain total variation regularization. The optimal solutions to the variational problem require ``skip connections" between layers, which provides a new theoretical explanation for the benefits skip connections provide in practice \cite{he2016deep}.  The sparse nature of our solutions sheds new light on the roles of sparsity and redundancy in deep learning, ranging from ``drop-out" \cite{hinton2012improving} to the ``lottery ticket hypothesis" \cite{frankle2018lottery}. And finally, low-rank weight matrices are a natural by-product of our variational theory that has precedent in practical studies of deep neural networks; it has been empirically observed that low-rank weight matrices can speed up learning \cite{ba2014deep} and improve accuracy \cite{golubeva2020wider}, robustness \cite{sanyal2019stable}, and computational efficiency  \cite{wang2021pufferfish} of deep neural networks.

\subsection{Related work}
Viewing regularized neural network training problems as variational problems over certain function spaces has received a lot of interest in the last few years~\cite{convex-nn,relu-linear-spline,min-norm-nn-splines,ridge-splines,representer-deep, deep-splines-lipschitz, deep-splines}, although many of the techniques used in these works are quite classical and rooted in variational spline theory and the study of continuous-domain inverse problems~\cite{rep-thm-radon-measure-recovery, fisher-jerome, locally-adaptive-regression-splines}. A common theme in these works is to leverage the sparisfying nature of total variation (TV) regularization to learn \emph{sparse solutions}. Our previous work in~\cite{min-norm-nn-splines,ridge-splines} proves representer theorems for both univariate and multivariate single-hidden layer neural networks by considering such sparsity-promoting total variation regularization. The key analysis tool used in~\cite{ridge-splines} was the Radon transform due to its tight connections with the analysis of ridge functions. This is because single-hidden layer neural networks are  superpositions of ridge functions (neurons). While the connections between ridge functions and the Radon transform is classical, dating back to early work in representation of solutions to certain partial differential equations as superpositions of ridge functions~\cite{plane-waves-pdes}, working with single-hidden layer ReLU networks in the Radon domain was first studied by~\cite{function-space-relu}.

Another line of related work is concerned with the ``optimal shaping'' of the activation functions in a deep neural network~\cite{representer-deep, deep-splines-lipschitz, deep-splines}. In particular,~\cite{representer-deep} proves a representer theorem regarding the optimal shaping of the activation functions. They consider the standard fully-connected feedforward deep neural network architecture, but allow the activation functions to be learnable. They impose a second-order total variation penalty on the activation functions and so the optimal shaping of the activation functions corresponds to linear splines with adaptive knot locations. We remark that we use several techniques developed in~\cite{representer-deep, deep-splines-lipschitz} to prove our representer theorem in this paper, particularly in proving existence of solutions to the variational problem we study. Finally, there is a line of work regarding ``deep kernel learning''~\cite{deep-kernel-learning}, in which they derive a representer theorem for compositions of kernel machines. They consider a construction similar to ours regarding the function space they study, but they consider compositions of reproducing kernel Hilbert spaces and so the resulting solutions to their variational problem do not take the form of a deep neural network.

\subsection{Roadmap}
In \Cref{sec:prelim} we introduce the notation and mathematical formulation used in the remainder of the paper as well as extend the results of~\cite{ridge-splines} in preparation for proving our deep ReLU network representer theorem. In \Cref{sec:rep-thm} we prove our main result, the representer theorem for deep ReLU networks. In \Cref{sec:nn-apps} we discuss applications of our representer theorem to the training and regularization of deep ReLU networks.

\section{Preliminaries} \label{sec:prelim}
Let $\Sch(\R^d)$ be the Schwartz space of smooth and rapidly decaying test functions on $\R^d$.  Its continuous dual, $\Sch'(\R^d)$, is the space of tempered distributions on $\R^d$. We are also interested in these spaces on $\cyl$, where $\Sph^{d-1}$ denotes the surface of the Euclidean sphere in $\R^d$.  We say $\psi \in \Sch(\cyl)$ when $\psi$ is smooth and satisfies the decay condition
\[
  \sup_{\substack{\vec{\gamma} \in \Sph^{d-1} \\ t \in \R}}
  \abs{\paren{1 +
  \abs{t}^k} \de[^\ell]{t^\ell} (\D\psi)(\vec{\gamma}, t)} < \infty,
\]
for all integers $k, \ell \geq 0$ and for all differential operators of all orders $\D$ in $\vec{\gamma}$~\cite[Chapter~6]{fourier}. Since the Schwartz spaces are nuclear, it follows that the above definition is equivalent to saying $\Sch(\cyl) = \mathcal{D}(\Sph^{d-1}) \,\hat{\otimes}\, \Sch(\R)$, where $\mathcal{D}(\Sph^{d-1})$ is the space of smooth functions on $\Sph^{d-1}$ and $\hat{\otimes}$ is the \emph{topological} tensor product~\cite[Chapter~III]{tvs}. We can then define the space of tempered distributions on $\cyl$ as its continuous dual, $\Sch'(\cyl)$.

Let $X$ be a locally compact Hausdorff space. The Riesz--Markov--Kakutani representation theorem says that $\M(X)$, the Banach space of finite Radon measures on $X$, is the continuous dual of $C_0(X)$, the space of continuous functions vanishing at infinity~\cite[Chapter~7]{folland}. Since $C_0(X)$ is a Banach space when equipped with the uniform norm, we have
\begin{equation}
  \norm{u}_{\M(X)}
  \coloneqq \sup_{\substack{\varphi \in C_0(X) \\
  \norm{\varphi}_\infty = 1}} \ang{u, \varphi}.
\label{eq:M-norm}
\end{equation}

The norm $\norm{\dummy}_{\M(X)}$ is exactly the \emph{total variation} norm (in the sense of measures). As $\Sch(X)$ is dense in $C_0(X)$, we can associate every measure in $\M(X)$ with a tempered distribution and view $\M(X) \subset \Sch'(X)$, providing the description
\[
  \M(X) \coloneqq \curly{u \in \Sch'(X)
  \st \norm{u}_{\M(X)} < \infty},
\]
and so the duality pairing $\ang{\dummy, \dummy}$ in \cref{eq:M-norm} can be viewed, formally, as the integral
\[
    \ang{u, \varphi} = \int_{X} \varphi(\vec{x}) u(\vec{x}) \dd \vec{x},
\]
where $u$ is viewed as an element of $\Sch'(X)$. The space $\M(X)$ can be viewed as a ``generalization'' of $L^1(X)$ in the sense that for any $f \in L^1(X)$, $\norm{f}_{L^1(X)} = \norm{f}_{\M(X)}$, but $\M(X)$ is a strictly larger space that also includes the shifted Dirac impulses $\delta(\dummy - \vec{x}_0)$, $\vec{x}_0 \in X$, with the property that $\norm{\delta(\dummy - \vec{x}_0)}_{\M(X)} = 1$. We also remark that the $\M$-norm is the continuous-domain analogue of the $\ell^1$-norm. In this paper, we will mostly work with $X = \cyl$.

\subsection{Scalar-valued single-hidden layer ReLU networks and variational problems}
Our work in~\cite{ridge-splines} proved a representer theorem for single-hidden layer ReLU networks with scalar outputs by considering variational problems over the space of functions of second-order bounded variation in the Radon domain. The Radon transform of a function $f: \R^d \to \R$ is given by
\[
    \Radon{f}(\vec{\gamma}, t) \coloneqq \int_{\curly{\vec{x}: \vec{\gamma}^\T
  \vec{x} = t}} f(\vec{x}) \dd s(\vec{x}), \quad (\vec{\gamma},t) \in \cyl,
\]
where $s$ denotes the surface measure on the hyperplane $\curly{\vec{x}
\st \vec{\gamma}^\T \vec{x} = t}$. The Radon domain is parameterized by a \emph{direction} $\vec{\gamma} \in \Sph^{d-1}$ and an \emph{offset} $t \in \R$. When working with the Radon transform of functions defined on $\R^d$, the following \emph{ramp filter} arises in the Radon inversion formula
\[
    \Lambda^{d-1} = (-\partial_t^2)^{\frac{d-1}{2}},
\]
where $\partial_t$ denotes the partial derivative with respect to the offset variable, $t$, of the Radon domain and fractional powers are defined in terms of Riesz potentials. The space of functions of second-order bounded variation in the Radon domain is then given by
\begin{equation}
    \RBV^2(\R^d) = \curly{f \in L^{\infty, 1}(\R^d) \st \RTV^2(f) < \infty},
    \label{eq:RBV}
\end{equation}
where $L^{\infty, 1}(\R^d)$ is the Banach space\footnote{It is a Banach space when equipped with the norm $\norm{f}_{\infty, 1} \coloneqq \esssup_{\vec{x} \in \R^d} \abs{f(\vec{x})}(1 + \norm{\vec{x}}_2)^{-1}$.} of functions mapping $\R^d \to \R$ of at most linear growth and
\begin{equation}
    \RTV^2(f) = c_d \norm*{\partial_t^2 \ramp^{d-1} \RadonOp f}_{\M(\cyl)}
    \label{eq:RTV}
\end{equation}
denotes the second-order total variation of a function in the offset variable of the Radon domain, where $c_d^{-1} = 2(2\pi)^{d-1}$ is a dimension-dependant constant that arises when working with the Radon transform. Note that all the operators that appear in \cref{eq:RTV} must be understood in the distributional sense. We refer the reader to~\cite[Section~3]{ridge-splines} for more details. We now state the main result of~\cite{ridge-splines}.

\begin{proposition}[{special case of~\cite[Theorem~1]{ridge-splines}}] \label{prop:ridge-spline-rep-thm}
    Consider the problem of interpolating the scattered data $\curly{(\vec{x}_n, y_n)}_{n=1}^N \subset \R^d \times \R$ with $N > d + 1$. Then, under the hypothesis of feasibility (i.e., $y_n = y_m$ whenever $\vec{x}_n = \vec{x}_m$), there exists a solution to the variational problem
    \begin{equation}
        \min_{f \in \RBV^2(\R^d)} \RTV^2(f) \quad\subj\quad f(\vec{x}_n) = y_n, \: n = 1, \ldots, N
        \label{eq:ridge-spline-variational}
    \end{equation}
    of the form
    \begin{equation}
        s(\vec{x}) = \sum_{k=1}^K v_k \, \rho(\vec{w}_k^\T\vec{x} - b_k) + \vec{c}^\T\vec{x} + c_0,
        \label{eq:ridge-spline}
    \end{equation}
    where $K \leq N - (d + 1)$, $\rho = \max\curly{0, \dummy}$, $v_k \in \R$, $\vec{w}_k \in \Sph^{d-1}$, $b_k \in \R$, $\vec{c} \in \R^d$, and $c_0 \in \R$.
\end{proposition}

\begin{remark}
    \cref{prop:ridge-spline-rep-thm} says that there always exists a solution to the variational problem in \cref{eq:ridge-spline-variational} that can realizable by a single-hidden layer ReLU network with a \emph{skip connection}~\cite{he2016deep}, which is the affine term in \cref{eq:ridge-spline}. In other words, \cref{prop:ridge-spline-rep-thm} is a \emph{representer theorem} for single-hidden layer ReLU networks.
\end{remark}

\begin{remark} \label{rem:rescale}
    As discussed in~\cite[Remark~3]{ridge-splines}, the fact that $\vec{w}_k \in \Sph^{d-1}$ in \cref{eq:ridge-spline} does not restrict the single-hidden layer neural network due to the positive homogeneity of the ReLU. Indeed, given any single-hidden layer neural network with $\vec{w}_k \in \R^d \setminus \curly{\vec{0}}$, we can use the fact that ReLU is positively homogeneous of degree $1$ to rewrite the network as
    \[
        \vec{x} \mapsto \sum_{k=1}^K v_k \norm{\vec{w}_k}_2 \, \rho_m(\tilde{\vec{w}}_k^\T \vec{x} - \tilde{b}_k) + \vec{c}^\T\vec{x} + c_0,
    \]
    where $\tilde{\vec{w}}_k \coloneqq \vec{w}_k / \norm{\vec{w}_k}_2 \in \Sph^{d-1}$ and $\tilde{b}_k \coloneqq b_k / \norm{\vec{w}_k}_2 \in \R$.
\end{remark}

Given a single-hidden layer ReLU network, we can explicitly compute its $\RTV^2$-seminorm in terms of network parameters. This is summarized in the following proposition. 

\begin{proposition}[{special case of~\cite[Lemma~25]{ridge-splines}}] \label{prop:nn-seminorm}
    Given a single-hidden layer neural network
    \[
        s(\vec{x}) = \sum_{k=1}^K v_k \, \rho(\vec{w}_k^\T\vec{x} - b_k) + \vec{c}^\T\vec{x} + c_0,
    \]
    where $\rho = \max\curly{0, \dummy}$, $v_k \in \R$, $\vec{w}_k \in \R^d$, $b_k \in \R$, $\vec{c} \in \R^d$, and $c_0 \in \R$,
    \begin{equation}
        \RTV^2(s) = \sum_{k=1}^K \abs{v_k}\norm{\vec{w}_k}_2.
        \label{eq:RTV-nn}
    \end{equation}
\end{proposition}
We remark that \cref{eq:RTV-nn} is sometimes referred to as the \emph{path-norm} of the network~\cite{path-norm}. Moreover, we see that \cref{eq:RTV-nn} is a kind of $\ell^1$-norm on the network parameters, giving insight into the sparsity-promoting aspect of the $\RTV^2$-seminorm on network weights.

Note that $\RBV^2(\R^d)$ is defined by a seminorm, and the null space of $\RTV^2(\dummy)$ is nontrivial; it is the space of affine functions on $\R^d$. It was proven in~\cite[Theorem~22]{ridge-splines} that $\RBV^2(\R^d)$ can be turned into a \emph{bona fide} Banach space when equipped with an appropriate norm.

\begin{lemma} \label{thm:RBV-top-props}
    The space $\RBV^2(\R^d)$ equipped with the norm
    \begin{equation}
        \norm{f}_{\RBV^2(\R^d)} \coloneqq \RTV^2(f) + \abs{f(\vec{0})} + \sum_{k=1}^d \abs{f(\vec{e}_k) - f(\vec{0})},
        \label{eq:RBV-norm}
    \end{equation}
    where $\curly{\vec{e}_k}_{k=1}^d$ denotes the canonical basis of $\R^d$, has the following properties:
    \begin{enumerate}
        \item It is a Banach space. \label{item:RBV-Banach}
        \item For any $\vec{x}_0 \in \R^d$, the Dirac impulse $\delta(\dummy - \vec{x}_0): f \mapsto f(\vec{x}_0)$ is weak$^*$ continuous on $\RBV^2(\R^d)$. \label{item:weak*-continuous}
    \end{enumerate}
\end{lemma}

The proof of \cref{thm:RBV-top-props} appears in \cref{app:RBV-top-props}. We remark that \cref{item:RBV-Banach} is a corollary of~\cite[Theorem~22]{ridge-splines} and \cref{item:weak*-continuous} is a new result. In particular, \cref{item:weak*-continuous} plays a crucial role in proving existence of solutions to the variational problem studied in our deep ReLU network representer theorem. The $\RBV^2(\R^d)$-norm is a sparsity promoting norm since $\RTV^2(\dummy)$ is defined via an $\M$-norm, the continuous-domain analogue of the $\ell^1$-norm.

\begin{remark} \label{rem:RTV-regularized}
    \Cref{thm:RBV-top-props} implies that the result of \cref{prop:ridge-spline-rep-thm} also holds for regularized problems of the form
    \[
        \min_{f \in \RBV^2(\R^d)} \: \sum_{n=1}^N \ell(y_n, f(\vec{x}_n)) + \lambda \, \RTV^2(f),
    \]
    where $\lambda > 0$ is an adjustable regularization parameter and the loss function $\ell(\dummy, \dummy)$ is convex, coercive, and lower semi-continuous. Note that these are slightly weaker conditions on the loss function than in~\cite[Theorem~1]{ridge-splines}. This version of the result holds due to the weak$^*$ continuity of the Dirac impulse $\delta(\dummy - \vec{x}_0): f \mapsto f(\vec{x}_0)$ on $\RBV^2(\R^d)$ combined with~\cite[Theorem~3]{unser2022convex} for the conditions on the loss function.
\end{remark}

While \cref{prop:ridge-spline-rep-thm} provides a powerful representer theorem result for single-hidden layer neural networks, the affine component of any solution is unregularized due to the null space of $\RTV^2(\dummy)$ being the space of affine functions on $\R^d$. Therefore, we modify the problem in \cref{eq:ridge-spline-variational} in order to explicitly regularize the affine component of the functions. This results in the following new representer theorem for single-hidden layer ReLU networks.

\begin{theorem} \label{thm:banach-rep-thm}
    Consider the problem of interpolating the scattered data $\curly{(\vec{x}_n, y_n)}_{n=1}^N \subset \R^d \times \R$ with $N > 0$. Then, under the hypothesis of feasibility (i.e., $y_n = y_m$ whenever $\vec{x}_n = \vec{x}_m$), there exists a solution to the variational problem
    \begin{equation}
        \min_{f \in \RBV^2(\R^d)} \norm{f}_{\RBV^2(\R^d)} \quad\subj\quad f(\vec{x}_n) = y_n, \: n = 1, \ldots, N
        \label{eq:banach-rep-variational}
    \end{equation}
    of the form
    \begin{equation}
        s(\vec{x}) = \sum_{k=1}^K v_k \, \rho(\vec{w}_k^\T\vec{x} - b_k) + \vec{c}^\T\vec{x} + c_0,
        \label{eq:banach-rep-soln}
    \end{equation}
    where $K \leq N$, $\rho = \max\curly{0, \dummy}$, $v_k \in \R$, $\vec{w}_k \in \Sph^{d-1}$, $b_k \in \R$, $\vec{c} \in \R^d$, and $c_0 \in \R$.
\end{theorem}

The proof of \cref{thm:banach-rep-thm} appears in \cref{app:banach-rep-thm}. The key difference between \cref{thm:banach-rep-thm} and \cref{prop:ridge-spline-rep-thm} is that in \cref{thm:banach-rep-thm}, we are minimizing the $\RBV^2(\R^d)$-norm rather than the $\RTV^2$-seminorm as in \cref{prop:ridge-spline-rep-thm}. This results in the sparsity of the number of neurons in the solution being $N$ rather than $N - (d + 1)$. Additionally, \cref{thm:banach-rep-thm} explicitly regularizes the skip connection that appears in \cref{eq:banach-rep-soln}.

\begin{lemma} \label{lemma:nn-norm}
    Given a single-hidden layer neural network
    \[
        s(\vec{x}) = \sum_{k=1}^K v_k \, \rho(\vec{w}_k^\T\vec{x} - b_k) + \vec{c}^\T\vec{x} + c_0,
    \]
    where $\rho = \max\curly{0, \dummy}$, $v_k \in \R$, $\vec{w}_k \in \R^d$, $b_k \in \R$, $\vec{c} \in \R^d$, and $c_0 \in \R$,
    \begin{equation}
        \norm{s}_{\RBV^2(\R^d)} = \sum_{k=1}^K \abs{v_k}\norm{\vec{w}_k}_2 + \abs{s(\vec{0})} + \sum_{n=1}^d \abs{s(\vec{e}_n) - s(\vec{0})}.
        \label{eq:RBV-nn-norm}
    \end{equation}
\end{lemma}
\begin{proof}
    The result follows from \cref{prop:nn-seminorm,thm:RBV-top-props}.
\end{proof}


\subsection{Vector-valued single-hidden layer ReLU networks and variational problems}
Since a deep neural network is the composition of vector-valued single-hidden layer neural networks, we require a representer theorem for vector-valued single-hidden layer ReLU networks as a precursor to our representer theorem for deep ReLU networks. Extending \cref{thm:banach-rep-thm} for vector-valued functions follows standard techniques. In particular, we follow the technique of~\cite{vector-valued-smoothing-splines} which derives a representer theorem for vector-valued smoothing splines.

\begin{lemma} \label{lemma:vv-top-props}
    Define the vector-valued analogue of $\RBV^2(\R^d)$ by the Cartesian product
    \[
        \underbrace{\RBV^2(\R^d) \times \cdots \times \RBV^2(\R^d)}_\text{$D$ times}.
    \]
    This space can be viewed as the \emph{Bochner space} $\ell^1([D]; \RBV^2(\R^d))$, where $[D] = \curly{1, \ldots, D}$,
    and can therefore be equipped with the norm
    \[
         \norm{f}_{\ell^1([D]; \RBV^2(\R^d))} = \sum_{m=1}^D \norm{f_m}_{\RBV^2(\R^d)},
    \]
    where $f = (f_1, \ldots, f_D)$. For brevity, write $\RBV^2(\R^d; \R^D)$ for $\ell^1([D]; \RBV^2(\R^d))$. This space has the following properties:
    \begin{enumerate}
        \item It is a Banach space. \label{item:vv-banach}
        \item For any $\vec{x}_0 \in \R^d$, the point evaluation operator 
        \[
            \tilde{\vec{x}}_0: f \mapsto  f(\vec{x}_0) = \begin{bmatrix}
                \ang{\delta(\dummy - \vec{x}_0), f_1} \\
                \vdots \\
                \ang{\delta(\dummy - \vec{x}_0), f_D}
            \end{bmatrix}
            =
            \begin{bmatrix}
                f_1(\vec{x}_0) \\ \vdots \\ f_D(\vec{x}_0)
            \end{bmatrix}
        \]
         is component-wise weak* continuous. \label{item:vv-weak*-continuous}
    \end{enumerate}
\end{lemma}
\begin{proof}
    \Cref{item:vv-banach} follows by construction since $\RBV^2(\R^d)$ is itself a Banach space from \cref{item:RBV-Banach} in \cref{thm:RBV-top-props}. \Cref{item:weak*-continuous} follows from \cref{item:weak*-continuous} in \cref{thm:RBV-top-props}.
\end{proof}
\begin{remark}
    We can define different (but equivalent) norms on $\RBV^2(\R^d; \R^D)$ via the $\ell^p([D]; \RBV^2(\R^d))$-norms, where $1 \leq p < \infty$. We focus on the case of $p = 1$ in this paper for clarity.
\end{remark}

\begin{lemma} \label{lemma:Lipschitz-bound}
    Let $f \in \RBV^2(\R^d; \R^D)$. Then, $f$ is Lipschitz continuous and satisfies the Lipschitz bound
    \[
        \norm{f(\vec{x}) - f(\vec{y})}_1 \leq \norm{f}_{\RBV^2(\R^d; \R^D)} \, \norm{\vec{x} - \vec{y}}_1.
    \]
\end{lemma}

The proof of \cref{lemma:Lipschitz-bound} appears in \cref{app:Lipschitz-bound}.

\begin{theorem} \label{thm:vv-banach-rep-thm}
    Consider the problem of interpolating the scattered data $\curly{(\vec{x}_n, \vec{y}_n)}_{n=1}^N \subset \R^d \times \R^D$ with $N > 0$. Then, under the hypothesis of feasibility (i.e., $\vec{y}_n = \vec{y}_m$ whenever $\vec{x}_n = \vec{x}_m$), there exists a solution to the variational problem
    \begin{equation}
        \min_{f \in \RBV^2(\R^d; \R^D)} \norm{f}_{\RBV^2(\R^d; \R^D)} \quad\subj\quad f(\vec{x}_n) = \vec{y}_n, \: n = 1, \ldots, N
        \label{eq:vv-banach-rep-variational}
    \end{equation}
    of the form
    \begin{equation}
        s(\vec{x}) = \sum_{k=1}^K \vec{v}_k \, \rho(\vec{w}_k^\T\vec{x} - b_k) + \mat{C}\vec{x} + \vec{c}_0,
        \label{eq:vv-banach-rep-soln}
    \end{equation}
    where $K \leq ND$, $\rho = \max\curly{0, \dummy}$, $\vec{v}_k \in \R^D$, $\vec{w}_k \in \Sph^{d-1}$, $b_k \in \R$, $\mat{C} \in \R^{D \times d}$, and $\vec{c}_0 \in \R^D$. Moreover, there always exists a solution of the form in \cref{eq:vv-banach-rep-soln} in which $\vec{v}_k$ is $1$-sparse.
\end{theorem}

The proof of \cref{thm:vv-banach-rep-thm} appears in \cref{app:vv-banach-rep-thm}. We also remark that the tightness of the bound $K \leq ND$ is an open question.

\begin{remark} \label{rem:vv-rescale}
    As discussed in \cref{rem:rescale}, the fact that $\vec{w}_k \in \Sph^{d-1}$ in \cref{eq:vv-banach-rep-soln} does not restrict the single-hidden layer neural network due to the positive homogeneity of the ReLU.
\end{remark}
\begin{lemma} \label{lemma:vv-nn-norm}
    Given a vector-valued single-hidden layer neural network
    \[
        s(\vec{x}) = \sum_{k=1}^K \vec{v}_k \, \rho(\vec{w}_k^\T\vec{x} - b_k) + \mat{C}\vec{x} + \vec{c}_0,
    \]
    where $\rho = \max\curly{0, \dummy}$, $\vec{v}_k \in \R^D$, $\vec{w}_k \in \R^d$, $b_k \in \R$, $\mat{C} \in \R^{D \times d}$, and $\vec{c}_0 \in \R^D$,
    \begin{equation}
        \norm{s}_{\RBV^2(\R^d; \R^D)} = \sum_{k=1}^K \norm{\vec{v}_k}_1 \norm{\vec{w}_k}_2 + 
        \sum_{m=1}^D \paren{\abs{s_m(\vec{0})} + \sum_{n=1}^d \abs{s_m(\vec{e}_n) - s_m(\vec{0})}}.
        \label{eq:vv-RBV-nn-norm}
    \end{equation}
\end{lemma}
\begin{proof}
    For $m = 1, \ldots, D$, we can write
    \[
        s_m(\vec{x}) = \sum_{k=1}^K v_{k,m} \,\rho(\vec{w}_k^\T\vec{x} - b_k) + \vec{c}_m^\T\vec{x} + c_{0, m},
    \]
    where $s_m$ is the $m$th component of $s$, $\vec{c}_m$ is the $m$th row of $\mat{C}$, and $c_{0,m}$ is the $m$th component of $\vec{c}_0$. The result follows from \cref{lemma:nn-norm} and the definition of the $\RBV^2(\R^d; \R^D)$-norm.
\end{proof}
    
\section{A Representer Theorem for Deep ReLU Networks} \label{sec:rep-thm}
In this section, we will prove our representer theorem for deep ReLU networks. We consider functions that are compositions of functions from the Banach spaces defined in \cref{lemma:vv-top-props}. Let
\begin{align*}
    &\RBV_{\mathsf{deep}}^2(\R^{d_0}; \cdots; \R^{d_L}) \\
    &\qquad\qquad \coloneqq \curly{f = f^{(L)} \circ \cdots \circ f^{(1)} \st f^{(\ell)} \in \RBV^2(\R^{d_{\ell - 1}}; \R^{d_\ell}), \ell = 1, \ldots, L}.
\end{align*}
denote the space of all such functions.

For brevity, we will write $\RBV_{\mathsf{deep}}^2(L)$ for $\RBV_{\mathsf{deep}}^2(\R^{d_0}; \cdots; \R^{d_L})$.
This definition reflects two standard architectural specifications for deep neural networks: the number of hidden layers $L$ and the functional ``widths'', $d_\ell$, of each layer. That is, each function in the composition will ultimately correspond to a layer in a deep neural network in our representer theorem.

\begin{lemma} \label{lemma:deep-Lipschitz-bound}
    Let $f = f^{(L)} \circ \cdots \circ f^{(1)} \in \RBV_{\mathsf{deep}}^2(L)$. Then, $f$ is Lipschitz continuous and satisfies the Lipschitz bound
    \[
        \norm{f(\vec{x}) - f(\vec{y})}_1 \leq \paren{\prod_{\ell=1}^L \norm*{f^{(\ell)}}_{\RBV^2(\R^{d_{\ell-1}}; \R^{d_\ell})}} \, \norm{\vec{x} - \vec{y}}_1.
    \]
\end{lemma}
\begin{proof}
    The result follows by repeatedly applying \cref{lemma:Lipschitz-bound}.
\end{proof}

We now state our representer theorem for deep ReLU networks.

\begin{theorem} \label{thm:deep-representer}
    Let $L$ be a positive integer corresponding to the depth of a deep ReLU network and let $d_0, \ldots, d_L$ be positive integers corresponding to the intermediate dimensions of a deep neural network. Consider the problem of approximating the scattered data $\curly{(\vec{x}_n, \vec{y}_n)}_{n=1}^N \subset \R^{d_0} \times \R^{d_L}$ with $N > 0$ denoting the number of data. Let $\ell(\dummy, \dummy)$ be an arbitrary nonnegative lower semi-continuous loss function and let $\lambda > 0$ be a regularization parameter. Then, there exists a solution to the variational problem
    \begin{equation}
        \min_{\substack{f^{(\ell)} \in \RBV^2(\R^{d_{\ell-1}}; \R^{d_\ell}) \\ \ell = 1, \ldots, L \\ f = f^{(L)} \circ \cdots \circ f^{(1)}}} \: \sum_{n=1}^N \ell(\vec{y}_n, f(\vec{x}_n)) + \lambda \sum_{\ell = 1}^L \norm*{f^{(\ell)}}_{\RBV^2(\R^{d_{\ell - 1}}; \R^{d_\ell})}
        \label{eq:deep-variational}
    \end{equation}
    of the form
    \begin{equation}
        s(\vec{x}) = \vec{x}^{(L)},
        \label{eq:deep-solution}
    \end{equation}
    where $\vec{x}^{(L)}$ is computed recursively via
    \begin{equation}
        \begin{cases}
            \vec{x}^{(0)} \coloneqq \vec{x}, \\
            \vec{x}^{(\ell)} \coloneqq \mat{V}^{(\ell)} \vec{\rho}(\mat{W}^{(\ell)} \vec{x}^{(\ell - 1)} - \vec{b}^{(\ell)}) + \mat{C}^{(\ell)}\vec{x}^{(\ell - 1)} + \vec{c}_0^{(\ell)}, & \ell = 1, \ldots, L,
        \end{cases}
        \label{eq:deep-architecture}
    \end{equation}
    where $\vec{\rho}$ applies $\rho = \max\curly{0, \dummy}$ component-wise and for $\ell = 1, \ldots, L$, $\mat{V}^{(\ell)} \in \R^{d_\ell \times K^{(\ell)}}$, $\mat{W}^{(\ell)} \in \R^{K^{(\ell)} \times d_{\ell - 1}}$, $\vec{b}^{(\ell)} \in \R^{K^{(\ell)}}$, $\mat{C}^{(\ell)} \in \R^{d_{\ell} \times d_{\ell - 1}}$, and $\vec{c}_0^{(\ell)} \in \R^{d_{\ell}}$, where $K^{(\ell)} \leq N d_\ell$.
\end{theorem}
\begin{remark}
    Note that the search space in \cref{eq:deep-variational} is over the Cartesian product
    \begin{equation}
        \RBV^2(\R^{d_0}; \R^{d_1})
        \times \cdots \times
        \RBV^2(\R^{d_{L - 1}}; \R^{d_L})
        \label{eq:cartesian-search-space}
    \end{equation}
    rather than $\RBV^2_\mathsf{deep}(L)$. This is because given a function $f \in \RBV^2_\mathsf{deep}(L)$, there could be many decompositions such that $f = f^{(L)} \circ \cdots \circ f^{(1)}$. Therefore, in order for the regularization term in \cref{eq:deep-variational} to be well-defined, we formulate the problem over \cref{eq:cartesian-search-space}.
\end{remark}
\begin{remark}
    \Cref{thm:deep-representer} also holds for the problem of interpolating scattered data.
\end{remark}

The neural network architecture that appears in \cref{eq:deep-architecture} can be seen in \cref{fig:deep-architecture}. Moreover, this exact architecture was recently studied in the empirical work in~\cite{golubeva2020wider}, and is referred to as a deep ReLU network with \emph{linear bottlenecks}.  Since the variational problem in \cref{eq:deep-variational} is reminiscent of the variational problems studied in variational spline theory and since the resulting deep ReLU network solution in \cref{eq:deep-solution} is a continuous piecewise-linear function,
in a similar vein to~\cite{representer-deep, deep-splines-lipschitz, deep-splines},
we refer to such functions as \emph{deep ridge splines} of degree one.

\begin{figure}[htbp]
    \centering
    \def\svgwidth{0.85\columnwidth}
\begingroup%
  \makeatletter%
  \providecommand\color[2][]{%
    \errmessage{(Inkscape) Color is used for the text in Inkscape, but the package 'color.sty' is not loaded}%
    \renewcommand\color[2][]{}%
  }%
  \providecommand\transparent[1]{%
    \errmessage{(Inkscape) Transparency is used (non-zero) for the text in Inkscape, but the package 'transparent.sty' is not loaded}%
    \renewcommand\transparent[1]{}%
  }%
  \providecommand\rotatebox[2]{#2}%
  \newcommand*\fsize{\dimexpr\f@size pt\relax}%
  \newcommand*\lineheight[1]{\fontsize{\fsize}{#1\fsize}\selectfont}%
  \ifx\svgwidth\undefined%
    \setlength{\unitlength}{508.30874634bp}%
    \ifx\svgscale\undefined%
      \relax%
    \else%
      \setlength{\unitlength}{\unitlength * \real{\svgscale}}%
    \fi%
  \else%
    \setlength{\unitlength}{\svgwidth}%
  \fi%
  \global\let\svgwidth\undefined%
  \global\let\svgscale\undefined%
  \makeatother%
  \begin{picture}(1,0.72962149)%
    \lineheight{1}%
    \setlength\tabcolsep{0pt}%
    \put(0,0){\includegraphics[width=\unitlength,page=1]{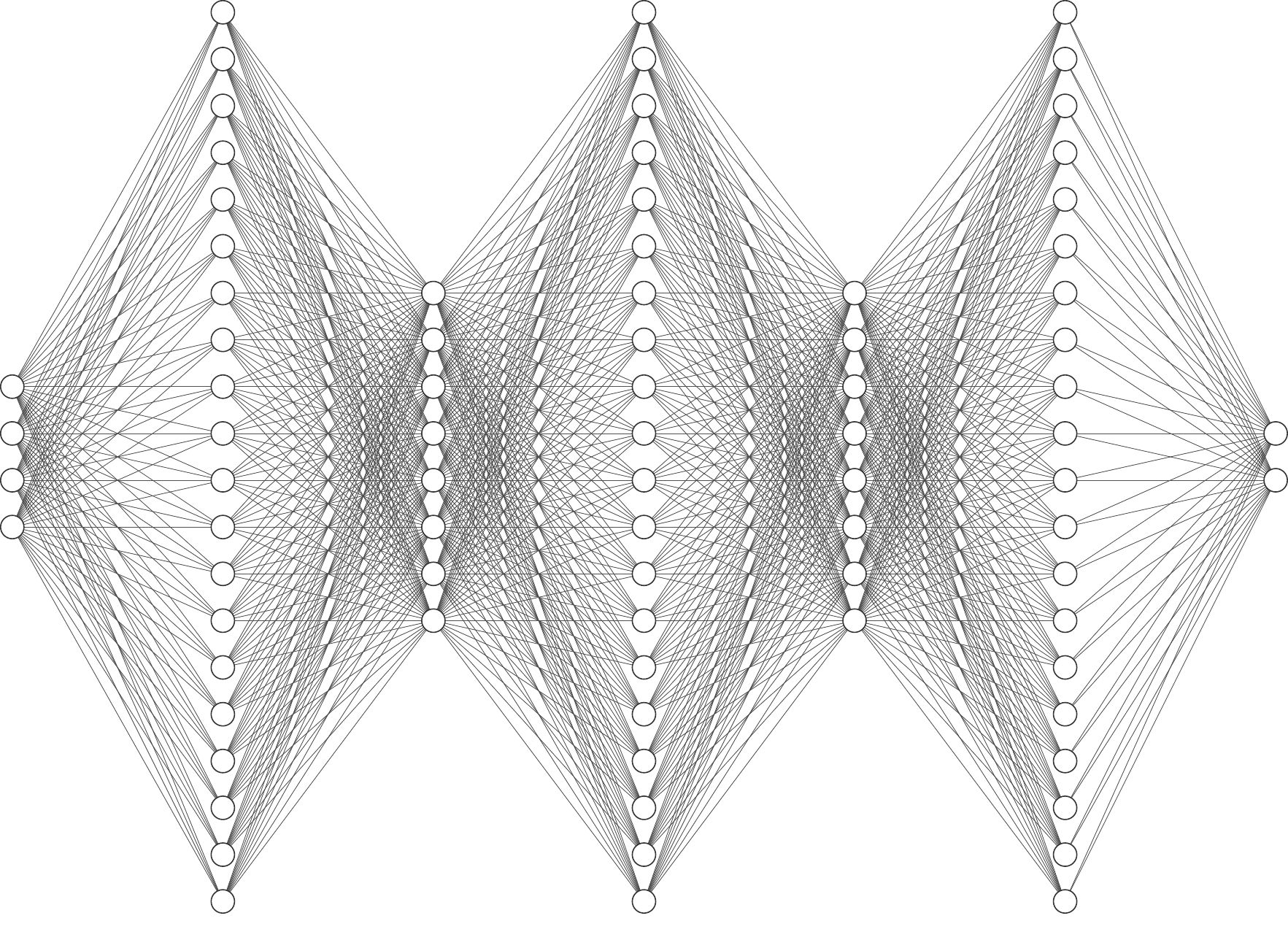}}%
    \put(0.07382618,0.01391909){\color[rgb]{0,0,0}\makebox(0,0)[lt]{\lineheight{1.25}\smash{\begin{tabular}[t]{l}$\vec{w}_k^{(1)}$\end{tabular}}}}%
    \put(0.72920782,0.01440395){\color[rgb]{0,0,0}\makebox(0,0)[lt]{\lineheight{1.25}\smash{\begin{tabular}[t]{l}$\vec{w}_k^{(3)}$\end{tabular}}}}%
    \put(0.40089443,0.01435514){\color[rgb]{0,0,0}\makebox(0,0)[lt]{\lineheight{1.25}\smash{\begin{tabular}[t]{l}$\vec{w}_k^{(2)}$\end{tabular}}}}%
    \put(0.23601285,0.01454219){\color[rgb]{0,0,0}\makebox(0,0)[lt]{\lineheight{1.25}\smash{\begin{tabular}[t]{l}$\vec{v}_k^{(1)}$\end{tabular}}}}%
    \put(0.5642077,0.01484379){\color[rgb]{0,0,0}\makebox(0,0)[lt]{\lineheight{1.25}\smash{\begin{tabular}[t]{l}$\vec{v}_k^{(2)}$\end{tabular}}}}%
    \put(0.8900825,0.01507928){\color[rgb]{0,0,0}\makebox(0,0)[lt]{\lineheight{1.25}\smash{\begin{tabular}[t]{l}$\vec{v}_k^{(3)}$\end{tabular}}}}%
    \put(0,0){\includegraphics[width=\unitlength,page=2]{nn.pdf}}%
  \end{picture}%
\endgroup%

    \caption{This figure shows the architecture of the deep neural network in \cref{eq:deep-architecture} in the case of $L = 3$ hidden layers. The black nodes denote input nodes, the blue nodes denote ReLU nodes, and the gray nodes denote linear nodes. Skip connection nodes are omitted for clarity.}
    \label{fig:deep-architecture}
\end{figure}

\begin{remark}
    Since the regularizer in \cref{eq:deep-variational} directly controls the $\RBV^2(\R^{d_{\ell - 1}}; \R^{d_\ell})$-norm of each layer, we see from \cref{lemma:Lipschitz-bound}, that the variational problem is essentially regularizing a bound on the Lipschitz constant of the function.
\end{remark}

\begin{remark} \label{rem:monotone-functions-regularizer}
    The regularizer that appears in \cref{eq:deep-variational} can be replaced by
    \[
        \psi_0\paren{\sum_{\ell = 1}^L \psi_\ell\paren{\norm*{f^{(\ell)}}_{\RBV^2(\R^{d_{\ell - 1}}; \R^{d_\ell})}}},
    \]
    where $\psi_\ell: [0,\infty) \to \R$, $\ell = 0, \ldots, L$ is a strictly increasing and convex function, and still admit a solution that takes the form of a deep neural network as in \cref{eq:deep-solution}. Thus, there are many choices of regularization that result in a representer theorem for deep ReLU networks.
\end{remark}

\begin{remark} \label{rem:standard-arch}
    Notice that \cref{eq:deep-solution} is precisely the standard $L$-hidden layer deep ReLU network architecture with \emph{rank bounded weight matrices} and \emph{skip connections}. Indeed, the weight matrix of the $\ell$th layer is $\mat{A}^{(\ell)} \coloneqq \mat{W}^{(\ell + 1)}\mat{V}^{(\ell)}$.
    More specifically, by dropping biases and skip connections for clarity, we see that $s(\vec{x})$ in \cref{eq:deep-solution} can be computed recursively as
    \begin{equation}
        \begin{cases}
            \tilde{\vec{x}}^{(0)} \coloneqq \vec{x}, \\
            \tilde{\vec{x}}^{(\ell)} \coloneqq \vec{\rho}(\mat{A}^{(\ell-1)}\tilde{\vec{x}}^{(\ell - 1)}), & \ell = 1, \ldots, L, \\
            s(\vec{x}) \coloneqq \mat{A}^{(L)}\tilde{\vec{x}}^{(L)},
        \end{cases}
        \label{eq:nn-arch-no-bias-no-skip}
    \end{equation}
    where
    \[
        \begin{cases}
            \mat{A}^{(0)} \coloneqq \mat{W}^{(1)}, \\
            \mat{A}^{(\ell)} \coloneqq \mat{W}^{(\ell+1)}\mat{V}^{(\ell)}, & \ell = 2, \ldots, L - 1, \\
            \mat{A}^{(L)} \coloneqq \mat{V}^{(L)}.
        \end{cases}
    \]
    From the dimensions of $\mat{V}^{(\ell)}$ and $\mat{W}^{(\ell)}$ in \cref{thm:deep-representer}, we see that for $\ell = 0, \ldots, L$, $\rank (\mat{A}^{(\ell)}) \leq \min \curly{N d_{\ell + 1}, d_{\ell}}$ and $\rank(\mat{A}^{(L)}) \leq d_L$. In a typical scenario, where the $\curly{d_\ell}_{\ell=1}^L$ are of the same order, this says that $\rank(\mat{A}^{(\ell)}) \leq d_\ell$.
\end{remark}
\begin{remark}
    The architecture of the network in \cref{eq:deep-architecture} is not restrictive of what functions can be represented by such a network. In particular, the architecture in \cref{eq:deep-architecture} is as expressive as the standard deep ReLU network architecture with hidden layer widths of $d_1, \ldots, d_{L}$.
\end{remark}

\begin{proof}[Proof of \cref{thm:deep-representer}]

    Given $f  = f^{(L)} \circ \cdots \circ f^{(1)}$ such that $f^{(\ell)} \in \RBV^2(\R^{d_{\ell-1}}; \R^{d_\ell})$, $\ell=1, \ldots, L$, write
    \[
        \J(f) \coloneqq \J(f^{(1)}, \ldots, f^{(L)}) \coloneqq \sum_{n=1}^N \ell(\vec{y}_n, f(\vec{x}_n)) + \lambda \sum_{\ell = 1}^L \norm*{f^{(\ell)}}_{\RBV^2(\R^{d_{\ell - 1}}; \R^{d_\ell})}
    \]
    for the objective value of $f$. Next, consider an arbitrary $g  = g^{(L)} \circ \cdots \circ g^{(1)}$ such that $g^{(\ell)} \in \RBV^2(\R^{d_{\ell-1}}; \R^{d_\ell})$, $\ell=1, \ldots, L$, with objective value $C \coloneqq \J(g)$.
    We may transform the unconstrained problem in \cref{eq:deep-variational} into the equivalent constrained problem
    \begin{equation}
        \min_{\substack{f^{(\ell)} \in \RBV^2(\R^{d_{\ell-1}}; \R^{d_\ell}) \\ \ell = 1, \ldots, L \\ f = f^{(L)} \circ \cdots \circ f^{(1)}}} \: \J(f)
        \quad\subj\quad \norm*{f^{(\ell)}}_{\RBV^2(\R^{d_{\ell - 1}}; \R^{d_\ell})} \leq C/\lambda, \: \ell = 1, \ldots, L.
        \label{eq:deep-constrained}
    \end{equation}
    This transformation is valid since any function that does not satisfy the constraints in \cref{eq:deep-constrained} has a strictly larger objective value than $g$, and is therefore not in the solution set.
    
    For $f_0  = f_0^{(L)} \circ \cdots \circ f_0^{(1)}$, $f_0^{(\ell)} \in \RBV^2(\R^{d_{\ell-1}}; \R^{d_\ell})$, $\ell=1, \ldots, L$, we will show that the map $f_0^{(\tilde{\ell})} \mapsto \J(f_0)$, for a fixed $\tilde{\ell} \in \curly{1, \ldots, L}$, is weak$^*$ lower semi-continuous on $\RBV^2(\R^{d_{\tilde{\ell} - 1}}; \R^{d_{\tilde{\ell}}})$. First notice that the map $f_0^{(\tilde{\ell})} \mapsto f_0(\vec{x}_0)$, for any $\vec{x}_0 \in \R^d$, is component-wise weak$^*$ continuous on $\RBV^2(\R^{d_{\tilde{\ell}-1}}; \R^{d_{\tilde{\ell}}})$.
    Indeed, since each $f_0^{(\ell)}$, $\ell=1, \ldots, L$, is component-wise continuous by \cref{lemma:Lipschitz-bound} and since the point evaluation is component-wise weak$^*$ continuous by \cref{lemma:vv-top-props}, the map $f_0^{(\tilde{\ell})} \mapsto f_0^{(L)} \circ \cdots \circ f_0^{(1)}(\vec{x}_0)$ is made up of compositions of component-wise continuous and component-wise weak$^*$ continuous functions, and is therefore itself component-wise weak$^*$ continuous on $\RBV^2(\R^{d_{\tilde{\ell}- 1}}; \R^{d_{\tilde{\ell}}})$. Next, since the loss function is lower semi-continuous and every norm is weak$^*$ continuous on its corresponding Banach space, we have that $f_0^{(\tilde{\ell})} \mapsto \J(f_0)$ is weak$^*$ lower semi-continuous on $\RBV^2(\R^{d_{\tilde{\ell}- 1}}; \R^{d_{\tilde{\ell}}})$. Therefore, $(f_0^{(1)}, \ldots, f_0^{(L)}) \mapsto \J(f)$ is weak$^*$ continuous over the Cartesian product in \cref{eq:cartesian-search-space}. Finally, by the Banach--Alaoglu theorem~\cite[Chapter~3]{rudin-functional}, the feasible set in \cref{eq:deep-constrained} is weak$^*$ compact. Thus, there exists a solution to \cref{eq:deep-constrained} (and subsequently \cref{eq:deep-variational}) by the Weierstrass extreme value theorem on general topological spaces~\cite[Chapter~5]{convex-functional-analysis}.
    
    Let $\tilde{s} = \tilde{s}^{(L)} \circ \cdots \circ \tilde{s}^{(1)}$ be a (not necessarily unique) solution to \cref{eq:deep-variational}. 
    By applying $\tilde{s}$ to each data point $\vec{x}_n$, $n = 1, \ldots, N$, we can recursively compute the intermediate vectors $\vec{z}_{n, \ell} \in \R^{d_{\ell}}$ as follows
    \begin{itemize}
        \item Initialize $\vec{z}_{n, 0} \coloneqq \vec{x}_n$.
        \item For each $\ell = 1, \ldots, L$, recursively update $\vec{z}_{n, \ell} \coloneqq \tilde{s}^{(\ell)}(\vec{z}_{n, \ell - 1})$.
    \end{itemize}
    The solution $\tilde{s}$ must satisfy
    \[
        \tilde{s}^{(\ell)} \in \argmin_{f \in \RBV^2(\R^{d_{\ell - 1}}; \R^{d_\ell})}  \norm{f}_{\RBV^2(\R^{d_{\ell - 1}}; \R^{d_\ell})} \quad\subj\quad f(\vec{z}_{n, \ell - 1}) = \vec{z}_{n, \ell}, \: n = 1, \ldots, N,
    \]
    for $\ell = 1, \ldots, L$. To see this, note that if the above display did not hold, it would contradict the optimality of $\tilde{s}$. By \cref{thm:vv-banach-rep-thm}, there always exists a solution to the above display that enforces the form of the solution in \cref{eq:deep-solution}.
\end{proof}

\section{Applications to Deep Network Training and Regularization} \label{sec:nn-apps}
In this section we will discuss applications of the representer theorem in \cref{thm:deep-representer} to the training and regularization of deep ReLU networks. Since \cref{thm:deep-representer} guarantees existence of a solution to the variational problem in \cref{eq:deep-variational} that is realizable by a deep ReLU network as in \cref{eq:deep-solution}, one can find a solution to the problem in \cref{eq:deep-variational} by finding a solution to a finite-dimensional deep network training problem.
\begin{lemma} \label{lemma:deep-nn-norm}
    Given a deep neural network $s = s^{(L)} \circ \cdots \circ s^{(1)}$ as in \cref{eq:deep-solution},
    \begin{align*}
        &\sum_{\ell = 1}^L\norm*{s^{(\ell)}}_{\RBV^2(\R^{d_{\ell - 1}}; \R^{d_\ell})} \\
        &\qquad=
        \sum_{\ell = 1}^L \paren{\sum_{k=1}^{K^{(\ell)}} \norm*{\vec{v}_k^{(\ell)}}_1 \norm*{\vec{w}_k^{(\ell)}}_2 + 
        \sum_{m=1}^D \paren{\abs*{s_m^{(\ell)}(\vec{0})} + \sum_{n=1}^d \abs*{s_m^{(\ell)}(\vec{e}_n) - s_m^{(\ell)}(\vec{0})}}}
    \end{align*}
    where $\vec{v}_k^{(\ell)}$ is the $k$th column of $\mat{V}^{(\ell)}$ and $\vec{w}_k^{(\ell)}$ is the $k$th row of $\mat{W}^{(\ell)}$.
\end{lemma}
\begin{proof}
    The proof follows by invoking \cref{lemma:vv-nn-norm} on each $s^{(\ell)}$, $\ell = 1, \ldots, L$.
\end{proof}

\Cref{lemma:deep-nn-norm} implies the following corollary to \cref{thm:deep-representer}.
\begin{corollary} \label{cor:nn-problem-path-norm}
    Let $\vec{\theta}$ denote the parameters of a deep neural network as in \cref{eq:deep-solution} and let $\Theta = \R^M$ denote the space of these parameters, where $M$ is the total number of scalar parameters in the network. Write $f_\vec{\theta}$ to denote a deep neural network parameterized by $\vec{\theta}$. Then, the solutions to the finite-dimensional neural network training problem
    \begin{equation}
        \begin{aligned}
        &\min_{\vec{\theta} \in \Theta} \: \sum_{n=1}^N \ell(\vec{y}_n, f_\vec{\theta}(\vec{x}_n)) \\
        &\qquad+ \lambda \sum_{\ell = 1}^L \paren{\sum_{k=1}^{K^{(\ell)}} \norm*{\vec{v}_k^{(\ell)}}_1 \norm*{\vec{w}_k^{(\ell)}}_2 +
        \sum_{m=1}^D \paren{\abs*{f_{\vec{\theta},m}^{(\ell)}(\vec{0})} + \sum_{n=1}^d \abs*{f_{\vec{\theta},m}^{(\ell)}(\vec{e}_n) - f_{\vec{\theta},m}^{(\ell)}(\vec{0})}}}
        \end{aligned}
        \label{eq:nn-problem-path-norm}
    \end{equation}
    where $\curly{(\vec{x}_n, \vec{y}_n)}_{n=1}^N \subset \R^{d_0} \times \R^{d_L}$ is a scattered data set, $\ell(\dummy, \dummy)$ is an arbitrary non-negative lower semi-continuous loss function, and $\lambda > 0$ is an adjustable regularization parameter, are solutions to \cref{eq:deep-variational} so long as $K^{(\ell)} \geq N d_\ell$.
\end{corollary}

We can also consider a different regularizer than the one in \cref{cor:nn-problem-path-norm} that results in a finite-dimensional neural network training problem equivalent to \cref{eq:nn-problem-path-norm}.

\begin{corollary} \label{cor:nn-problem-weight-decay}
    The solutions to
    \begin{equation}
        \begin{aligned}
        &\min_{\vec{\theta} \in \Theta} \: \sum_{n=1}^N \ell(\vec{y}_n, f_\vec{\theta}(\vec{x}_n)) \\
        &\qquad+ \lambda \sum_{\ell = 1}^L \paren{\frac{\norm*{\mat{V}^{(\ell)}}_{1,2}^2 + \norm*{\mat{W}^{(\ell)}}_{\mathsf{F}}^2}{2} +
        \sum_{m=1}^D \paren{\abs*{f_{\vec{\theta},m}^{(\ell)}(\vec{0})} + \sum_{n=1}^d \abs*{f_{\vec{\theta},m}^{(\ell)}(\vec{e}_n) - f_{\vec{\theta},m}^{(\ell)}(\vec{0})}}}
        \end{aligned}
        \label{eq:nn-problem-weight-decay}
    \end{equation}
    are also solutions to \cref{eq:nn-problem-path-norm}, where
    \[
        \norm*{\mat{V}^{(\ell)}}_{1,2}^2 \coloneqq \sum_{k=1}^{K^{(\ell)}} \norm*{\vec{v}_k^{(\ell)}}_1^2
    \]
    is the mixed $\ell^1\ell^2$-norm of $\mat{V}^{(\ell)}$ and $\norm{\dummy}_\mathsf{F}$ is the usual Frobenius norm of a matrix. Moreover, the solutions to \cref{eq:nn-problem-weight-decay} satisfy the property that $\norm*{\vec{v}_k^{(\ell)}}_1 = \norm*{\vec{w}_k^{(\ell)}}_2$, $\ell = 1, \ldots, L$, $k = 1, \ldots, K^{(\ell)}$.
\end{corollary}
\begin{proof}
    The $k$th neuron in the $\ell$th layer of a deep neural network as in \cref{eq:deep-solution} takes the form $\vec{x} \mapsto \vec{v}_k^{(\ell)}\rho({\vec{w}_k^{(\ell)}}^\T\vec{x} - b_k^{(\ell)})$. Due to the positive homogenity of the ReLU, $\vec{v}_k^{(\ell)}$ and $\vec{w}_k^{(\ell)}$ can be rescaled so that $\norm*{\vec{v}_k^{(\ell)}}_1 = \norm*{\vec{w}_k^{(\ell)}}_2$ without altering the function of the network. Therefore, minimizing $\norm*{\vec{v}_k^{(\ell)}}_1^2 + \norm*{\vec{w}_k^{(\ell)}}_2^2$ is achieved when $\norm*{\vec{v}_k^{(\ell)}}_1 = \norm*{\vec{w}_k^{(\ell)}}_2$. The result then follows from the fact that when $\norm*{\vec{v}_k^{(\ell)}}_1 = \norm*{\vec{w}_k^{(\ell)}}_2$ we have
    \[
        \frac{\norm*{\vec{v}_k^{(\ell)}}_1^2 + \norm*{\vec{w}_k^{(\ell)}}_2^2}{2} = \norm*{\vec{v}_k^{(\ell)}}_1 \norm*{\vec{w}_k^{(\ell)}}_2.
    \]
\end{proof}

\begin{remark}
    While the problems in \cref{eq:nn-problem-path-norm,eq:nn-problem-weight-decay} take the form of neural network training problems with new, principled forms of regularization, it's important to note that the problems are nonconvex, and our results say nothing about algorithms for actually solving the optimization problems.
\end{remark}

\begin{remark}
    Due to the sparsity-promoting nature of the $\RBV^2(\R^{d_{\ell - 1}}; \R^{d_\ell})$-norms, the regularizers that appear in \cref{eq:nn-problem-path-norm,eq:nn-problem-weight-decay} promote sparse (in the sense of the number of neurons) deep ReLU network solutions.
\end{remark}

\begin{remark}
    The term
    \[
        \sum_{m=1}^D \paren{\abs*{f_{\vec{\theta},m}^{(\ell)}(\vec{0})} + \sum_{n=1}^d \abs*{f_{\vec{\theta},m}^{(\ell)}(\vec{e}_n) - f_{\vec{\theta},m}^{(\ell)}(\vec{0})}}
    \]
    that appears in \cref{eq:nn-problem-path-norm,eq:nn-problem-weight-decay} simply imposes an $\ell^1$-norm on the coefficients of the affine part (i.e., the skip connection in the neural network realization) on each layer (see \cref{app:RBV-top-props}). Therefore, one may also consider the regularizers
    \begin{equation}
            \sum_{\ell = 1}^L \paren{\sum_{k=1}^{K^{(\ell)}} \norm*{\vec{v}_k^{(\ell)}}_1 \norm*{\vec{w}_k^{(\ell)}}_2 +
            \norm*{\mat{C}^{(\ell)}}_{1,1} + \norm*{\vec{c}_0^{(\ell)}}_1}
            \label{eq:path-with-skip-nn-norm}
    \end{equation}
    in place of \cref{eq:nn-problem-path-norm} or
    \begin{equation}
            \sum_{\ell = 1}^L \paren{\frac{\norm*{\mat{V}^{(\ell)}}_{1,2}^2 + \norm*{\mat{W}^{(\ell)}}_{\mathsf{F}}^2}{2} + \norm*{\mat{C}^{(\ell)}}_{1,1} + \norm*{\vec{c}_0^{(\ell)}}_1},
            \label{eq:weight-decay-with-skip-nn-norm}
    \end{equation}
    in place of \cref{eq:nn-problem-weight-decay}, where
    \[
        \norm{\mat{C}}_{1,1} \coloneqq \sum_{m=1}^D \sum_{k=1}^d \abs{c_{m, k}}
    \]
    denotes the mixed $\ell^1\ell^1$-norm of $\mat{C}$.
\end{remark}

\begin{remark} \label{rem:drop-skip-regularization}
    It is common in many deep learning papers to consider deep neural networks without biases and skip connections (see, e.g.,~\cite{path-norm,geometry-deep-learning,norm-based-capacity-control-deep,statistical-deep}). Since the term
    \begin{equation}
        \sum_{m=1}^D \paren{\abs*{f_{\vec{\theta},m}^{(\ell)}(\vec{0})} + \sum_{n=1}^d \abs*{f_{\vec{\theta},m}^{(\ell)}(\vec{e}_n) - f_{\vec{\theta},m}^{(\ell)}(\vec{0})}}
        \label{eq:skip-regularization}
    \end{equation}
    that appears in \cref{eq:nn-problem-path-norm,eq:nn-problem-weight-decay} simply imposes an $\ell^1$-norm on the coefficients of the affine part (i.e., the skip connection in the neural network realization) on each layer (see \cref{app:RBV-top-props}), the following two regularizers naturally arise from our variational framework in the case of a deep neural network with no biases or skip connections:
    \begin{equation}
        \sum_{\ell = 1}^L \sum_{k=1}^{K^{(\ell)}} \norm*{\vec{v}_k^{(\ell)}}_1 \norm*{\vec{w}_k^{(\ell)}}_2
        \label{eq:sum-of-path-nn-seminorm}
    \end{equation}
    or
    \begin{equation}
        \frac{1}{2}\sum_{\ell = 1}^L \norm*{\mat{V}^{(\ell)}}_{1,2}^2 + \norm*{\mat{W}^{(\ell)}}_{\mathsf{F}}^2,
        \label{eq:sum-of-squares-nn-seminorm}
    \end{equation}
    where \cref{eq:sum-of-path-nn-seminorm,eq:sum-of-squares-nn-seminorm} are in fact equivalent by the same argument as in the proof of \cref{cor:nn-problem-weight-decay}.
\end{remark}
\subsection{Connections to existing deep network regularization schemes.}
The regularizers that appear in \cref{eq:nn-problem-path-norm,eq:nn-problem-weight-decay,eq:path-with-skip-nn-norm,eq:weight-decay-with-skip-nn-norm,eq:sum-of-path-nn-seminorm,eq:sum-of-squares-nn-seminorm} are \emph{principled} regularizers for training deep ReLU networks. Moreover, the discussed regularizers are related to the well-known weight decay~\cite{weight-decay} and path-norm~\cite{path-norm} regularizers for deep ReLU networks.

Training a neural network with weight decay is one of the most common regularization schemes for deep ReLU networks. This corresponds to an $\ell^2$-norm regularization on all the network weights. The regularizer that appears in \cref{eq:sum-of-squares-nn-seminorm} almost takes the form of an $\ell^2$-norm of the network weights except that the regularization on the $\mat{V}^{(\ell)}$ is not the Frobenius norm. By considering a slightly different architecture than in \cref{eq:deep-architecture}, where it is imposed that the columns of $\mat{V}^{(\ell)}$ are $1$-sparse, the regularizer in \cref{eq:sum-of-squares-nn-seminorm} exactly corresponds to weight decay (since $\norm*{\mat{V}^{(\ell)}}_{1,2}^2 = \norm*{\mat{V}^{(\ell)}}_\mathsf{F}^2$ when the columns of $\mat{V}^{(\ell)}$ are $1$-sparse). Training this architecture with this regularizer still corresponds to finding a solution to the variational problem in \cref{eq:deep-variational} since it simply imposes that the outputs of each layer of the deep network are completely \emph{decoupled} (see \cref{rem:no-sharing}). The utility of \emph{not} considering such an architecture is to promote \emph{neuron sharing} between the outputs in each layer outputs of each layer.

Another common regularization scheme for deep ReLU networks is the path-norm  regularizer. In particular, several works~\cite{path-norm,geometry-deep-learning,norm-based-capacity-control-deep,statistical-deep} consider deep ReLU networks with no biases or skip connections mapping $\R^d \to \R$ of the form $s(\vec{x}) = x^{(L)}$, where $x^{(L)}$ is computed via
\begin{equation}
    \begin{cases}
        \vec{x}^{(0)} \coloneqq \vec{x}, \\
        \vec{x}^{(\ell)} \coloneqq \vec{\rho}(\mat{A}^{(\ell - 1)}\vec{x}^{(\ell - 1)}), & \ell = 1, \ldots, L, \\
        x^{(L)} \coloneqq {\vec{a}^{(L)}}^\T\vec{x}^{(L)},
    \end{cases}
    \label{eq:nn-arch-no-bias-no-skip-BK}
\end{equation}
where $\vec{\rho}$ denotes applying $\rho$ component-wise, $\mat{A}^{(0)} \in \R^{K^{(1)} \times d}$, $\mat{A}^{(\ell)} \in \R^{K^{(\ell+1)} \times K^{(\ell)}}$, $\ell = 1, \ldots, L - 1$, and $\vec{a}^{(L)} \in \R^{K^{(L)}}$. Note that \cref{eq:nn-arch-no-bias-no-skip-BK} is almost the same as the architecture in our framework if we drop biases and skip connections (see \cref{eq:nn-arch-no-bias-no-skip} in \cref{rem:standard-arch}).
These works then consider path-norm regularization of the form
\begin{equation}
    \sum_{k_L=1}^{K^{(L)}}\sum_{k_{L-1}=1}^{K^{(L-1)}} \cdots \sum_{k_1 = 1}^{K^{(1)}} \sum_{k_0=1}^d \abs*{a_{k_0, k_1}} \abs*{a_{k_1, k_2}} \cdots \abs*{a_{k_{L-1}, k_{L}}} \abs*{a_{k_{L}}},
    \label{eq:path-norm}
\end{equation}
where $a_{k_{\ell}, k_{\ell+1}}$ denotes the $(k_{\ell}, k_{\ell+1})$th entry in $\mat{A}^{(\ell)}$ and $a_{k_{L}}$ denotes the $k_{L}$th entry in $\vec{a}^{(L)}$. 

Consider regularizing a deep ReLU network (with no biases or skip connections) from our framework with the following regularizer\footnote{Where we drop the term in \cref{eq:skip-regularization} as discussed in \cref{rem:drop-skip-regularization}.}, which arises with a particular choice of $\curly{\psi_\ell}_{\ell=0}^L$ in \cref{rem:monotone-functions-regularizer},
\begin{equation}
    \prod_{\ell = 1}^L \sum_{k=1}^{K^{(\ell)}} \norm*{\vec{v}_k^{(\ell)}}_1 \norm*{\vec{w}_k^{(\ell)}}_2.
    \label{eq:product-of-paths}
\end{equation}
We have that \cref{eq:product-of-paths} is an upper bound on something that looks very similar to the path-norm in \cref{eq:path-norm}. Indeed, first notice that if we write the deep ReLU network from our framework in the form in \cref{eq:nn-arch-no-bias-no-skip}, we have
\begin{equation}
    \abs*{a_{k_{\ell}, k_{\ell+1}}} = \abs*{{\vec{v}_k^{(\ell)}}^\T{\vec{w}_k^{(\ell+1)}}} \leq \norm*{\vec{v}_k^{(\ell)}}_2 \norm*{\vec{w}_k^{(\ell+1)}}_2 \leq \norm*{\vec{v}_k^{(\ell)}}_1\norm*{\vec{w}_k^{(\ell+1)}}_2,
    \label{eq:upper-bound-path-norm-computation}
\end{equation}
where $a_{k_{\ell}, k_{\ell+1}}$ denotes the $(k_{\ell}, k_{\ell+1})$th entry in $\mat{A}^{(\ell)}$ as defined in \cref{rem:standard-arch}. Therefore,
\begin{align*}
    \prod_{\ell = 1}^L \sum_{k=1}^{K^{(\ell)}} \norm*{\vec{v}_k^{(\ell)}}_1 \norm*{\vec{w}_k^{(\ell)}}_2
    &= \sum_{k_L=1}^{K^{(L)}} \cdots \sum_{k_1=1}^{K^{(1)}} \norm*{\vec{w}_{k_1}^{(1)}}_2 \norm*{\vec{v}_{k_1}^{(1)}}_1 \norm*{\vec{w}_{k_2}^{(2)}}_2 \norm*{\vec{v}_{k_2}^{(2)}}_1 \cdots \norm*{\vec{w}_{k_L}^{(L)}}_2 \norm*{\vec{v}_{k_L}^{(L)}}_1 \\
    &\geq \sum_{k_L=1}^{K^{(L)}} \cdots \sum_{k_1=1}^{K^{(1)}} \norm*{\vec{w}_{k_1}^{(1)}}_2 \, \abs{a_{k_1, k_2}} \cdots \abs{a_{k_{L-1},k_L}} \, \norm*{\vec{v}_{k_L}^{(L)}}_1,
\end{align*}
where the last line holds from \cref{eq:upper-bound-path-norm-computation}. We see that the last line in the above display is the same as the path-norm in \cref{eq:path-norm}, apart from how it treats weights in the first and last layers. We also remark that the work in~\cite{statistical-deep} show that the path-norm in \cref{eq:path-norm} controls the Rademacher and Gaussian complexity of deep ReLU networks.

\section{Conclusion} \label{sec:conclusion}
In this paper we have proven a representer theorem for deep ReLU networks\footnote{As stated in the introduction, all our results are straightforward to generalize to any truncated power activation function.}. We have shown that deep ReLU networks with $L$-hidden layers, skip connections, and rank bounded weight matrices are solutions to a variational problem over compositions of functions in $\RBV^2$-spaces. This variational problem can be recast as a finite-dimensional neural network training problem with various choices of regularization. We have therefore derived several new, principled regularizers for deep ReLU networks. Moreover, these regularizers promote sparse solutions. We have shown that these new regularizers are related to the well-known weight decay and path-norm regularization schemes commonly used in the training of deep ReLU networks. The main followup question revolves around more understanding of the compositional space $\RBV_{\mathsf{deep}}^2(L)$. This entails first having further understanding of the $\RBV^2$-spaces. The function spaces studied in this paper are new and not classical and future work will be directed at understanding how these new spaces are related to classical function spaces studied in functional analysis.

\appendix

\section{Topological Properties of \texorpdfstring{$\RBV^2(\R^d)$}{RBV2(Rd)}} \label{app:RBV-top-props}
In this section we will prove \cref{thm:RBV-top-props}. We will rely on many results developed in~\cite{ridge-splines}. While the definition of $\RBV^2(\R^d)$ given in \cref{eq:RBV} is convenient from an intuitive perspective, it does not lend itself to analysis due to $\RTV^2(\dummy)$ being a seminorm with null space $\mathcal{P}_1(\R^d)$, the space of polynomials of degree at most $1$, i.e., affine functions in $\R^d$. Thus, we use the result of~\cite[Theorem~22]{ridge-splines} to characterize $\RBV^2(\R^d)$ as a Banach space. In particular,~\cite[Theorem~22]{ridge-splines} considers an arbitrary \emph{biorthogonal system} of $\mathcal{P}_1(\R^d)$ in order to equip $\RBV^2(\R^d)$ with a \emph{bona fide} norm.

\begin{definition} \label{defn:biorthogonal-system}
    Let $\N$ be a finite-dimensional space with $N_0 \coloneqq \dim \N$. The pair $(\vec{\phi}, \vec{p}) = \curly{(\phi_n, p_n)}_{n=0}^{N_0 - 1}$ is called a \emph{biorthogonal system} for $\mathcal{N}$ if $\vec{p} = \curly{p_n}_{n=0}^{N_0 - 1}$ is a basis of $\N$ and the ``boundary'' functionals $\vec{\phi} = \curly{\phi_n}_{n=0}^{N_0 - 1}$ with $\phi_n \in \N'$ (the continuous dual of $\N$) satisfy the biorthogonality condition $\ang{\phi_k, p_n} = \delta[k - n]$, $k, n = 0, \ldots, N_0 - 1$, where $\delta[\dummy]$ is the Kronecker impulse.
\end{definition}

\begin{proposition}[{see~\cite[Theorem~22, Item~3]{ridge-splines}}] \label{prop:RBV-Banach-abstract}
    Let $(\vec{\phi}, \vec{p})$ be a biorthogonal system for $\mathcal{P}_1(\R^d)$. Then, $\RBV^2(\R^d)$ equipped with the norm
    \[
        \norm{f}_{\RBV^2(\R^d)} = \RTV^2(f) + \norm{\vec{\phi}(f)}_1,
    \]
    where $\vec{\phi}(f) = (\ang{\phi_0, f}, \ldots, \ang{\phi_d, f}) \in \R^{d+1}$, is a Banach space.
\end{proposition}

We can now prove \cref{item:RBV-Banach} of \cref{thm:RBV-top-props}.
\begin{proof}[Proof of \cref{thm:RBV-top-props}, \cref{item:RBV-Banach}]
    By \cref{prop:RBV-Banach-abstract} it suffices to find a biorthogonal system $(\vec{\phi}, \vec{p})$ of $\mathcal{P}_1(\R^d)$ so that for every $f \in \RBV^2(\R^d)$ we have
    \begin{equation}
        \norm{\vec{\phi}(f)}_1 = \abs{f(\vec{0})} + \sum_{k=1}^d \abs{f(\vec{e}_k) - f(\vec{0})}.
        \label{eq:null-space-norm}
    \end{equation}
    
    Put $p_0(\vec{x}) \coloneqq 1$ and $p_k(\vec{x}) \coloneqq x_k$, $k = 1, \ldots, d$. Clearly $\vec{p}$ is a basis for $\mathcal{P}_1(\R^d)$. Put $\phi_0 \coloneqq \delta$ and $\phi_k \coloneqq \delta(\dummy - \vec{e}_k) - \delta$, $k = 1, \ldots, d$, where $\delta$ denotes the Dirac impulse on $\R^d$ and $\vec{e}_k$ denotes the $k$th canonical basis vector of $\R^d$. Then, $(\vec{\phi}, \vec{p})$ is a biorthogonal system for $\mathcal{P}_1(\R^d)$. Indeed, we have
    $\ang{\phi_0, p_0} = 1$ and $\ang{\phi_k, p_k} = p_k(\vec{e}_k) - p_k(\vec{0}) = 1 - 0 = 1$, $k = 1, \ldots, d$. We also have
    \begin{align*}
        \ang{\phi_0, p_k} &= p_k(\vec{0}) = 0, \quad k = 1, \ldots, d, \\
        \ang{\phi_k, p_0} &= p_0(\vec{e}_k) - p_0(\vec{0}) = 1 - 1 = 0, \quad k = 1, \ldots, d, \\
        \ang{\phi_k, p_n} &= p_n(\vec{e}_k) - p_n(\vec{0}) = 0 + 0 = 0, \quad k, n = 1, \ldots, d, \quad k \neq n.
    \end{align*}
    A computation shows that \cref{eq:null-space-norm} holds with this choice of biorthogonal system.
\end{proof}

In order to prove \cref{item:weak*-continuous} of \cref{thm:RBV-top-props}, we must show that the Dirac impulse, $\delta(\dummy - \vec{x}_0)$, $\vec{x}_0 \in \R^d$, is weak$^*$ continuous on $\RBV^2(\R^d)$. The following proposition characterizes the weak$^*$ continuous linear functionals on a Banach space.
\begin{proposition}[{see~\cite[Theorem~IV.20, pg. 114]{math-phys}}] \label{prop:weak*-continuous-predual}
    Let $\X$ be a Banach space. The only weak$^*$ continuous linear functionals on $\X'$ (the continuous dual of $\X$) are elements of $\X$.
\end{proposition}
Therefore, we must show that the Dirac impulse is contained in the predual of $\RBV^2(\R^d)$. Before we can prove this, we require an important result from~\cite{ridge-splines}. Recall from \cref{eq:RTV} that
\[
    \RTV^2(f) = c_d \norm*{\partial_t^2 \ramp^{d-1} \RadonOp f}_{\M(\cyl)}.
\]
Put $\ROp \coloneqq c_d \, \partial_t^2 \ramp^{d-1} \RadonOp$. As discussed in~\cite{ridge-splines}, for every $f \in \RBV^2(\R^d)$, $u \coloneqq \ROp f \in \M(\cyl)$ is always even, i.e., $u(\vec{\gamma}, t) = u(-\vec{\gamma}, -t)$. This means we have
\[
    \RTV^2(f) = \norm{\ROp f}_{\M(\P^d)},
\]
where $\P^d$ denotes the manifold of hyperplanes on $\R^d$. In particular, we can view $\M(\P^d)$ as the subspace of $\M(\cyl)$ with only even finite Radon measures. Indeed, this is due to the fact that every hyperplane takes the form $h_{(\vec{\gamma}, t)} \coloneqq \curly{\vec{x} \in \R^d \st \vec{\gamma}^\T\vec{x} = t}$ for some $(\vec{\gamma}, t) \in \cyl$ and $h_{(\vec{\gamma}, t)} = h_{(-\vec{\gamma}, -t)}$.

\begin{proposition}[{see~\cite[Lemma~21~and~Theorem~22]{ridge-splines}}] \label{prop:direct-sum-inverse}
    Let $(\vec{\phi}, \vec{p})$ be a biorthogonal system for $\mathcal{P}_1(\R^d)$. Then, every $f \in \RBV^2(\R^d)$ has the unique direct-sum decomposition
    \[
        f = \ROp^{-1}_\vec{\phi}\curly{u} + q,
        \label{eq:direct-sum-decomp}
    \]
    where $u = \ROp f \in \M(\P^d)$, $q = \sum_{k=0}^{d} \ang{\phi_k, f}p_k \in \mathcal{P}_1(\R^d)$, and 
    \begin{equation}
        \ROp^{-1}_\vec{\phi}: u \mapsto \int_{\cyl} g_\vec{\phi}(\dummy, \vec{z}) u(\vec{z}) \dd(\sigma \times \lambda)(\vec{z}),
        \label{eq:right-inverse}
    \end{equation}
    where $\sigma$ is the surface measure on $\Sph^{d-1}$ and $\lambda$ is the Lebesgue measure on $\R$ and
    \begin{equation}
        g_\vec{\phi}(\vec{x}, \vec{z}) =  r_\vec{z}(\vec{x}) - \sum_{k=0}^{d} p_k(\vec{x}) q_k(\vec{z}),
        \label{eq:kernel-of-inverse}
    \end{equation}
    where $r_\vec{z} = r_{(\vec{w}, b)} = \rho(\vec{w}^\T(\dummy) - b)$ and $q_k(\vec{z}) \coloneqq \ang{\phi_k, r_\vec{z}}$, where $\vec{z} = (\vec{w}, b) \in \cyl$ and $\rho$ denotes any Green's function of $\D^2$, the second derivative operator, e.g., $\rho = \max\curly{0, \dummy}$ (the ReLU) or $\rho = \abs{\dummy} / 2$.
\end{proposition}

The operator $\ROp^{-1}_\vec{\phi}$ defined in \cref{eq:right-inverse} has several useful properties (see~\cite[Theorem~22, Items~1~and~2]{ridge-splines}). In particular, it is a stable (i.e., bounded) right-inverse of $\ROp$, and when restricted to
\[
    \RBV^2_\vec{\phi}(\R^d) \coloneqq \curly{f \in \RBV^2(\R^d) \st \vec{\phi}(f) = \vec{0}},
\]
it is the \emph{bona fide} inverse of $\ROp$. The space $\RBV^2_\vec{\phi}(\R^d)$ is also a concrete transcription of the abstract quotient $\RBV^2(\R^d) / \mathcal{P}_1(\R^d)$.
We have that $\ROp: \RBV^2_\vec{\phi}(\R^d) \to \M(\P^d)$ is an isometric isomorphism with inverse given by $\ROp^{-1}_\vec{\phi}$. Additionally we have from \cref{prop:direct-sum-inverse} that $\RBV^2(\R^d) \cong \RBV^2_\vec{\phi}(\R^d) \oplus \mathcal{P}_1(\R^d)$, where $\RBV^2_\vec{\phi}(\R^d)$ is a Banach space when equipped with the norm $f \mapsto \norm{\ROp f}_{\M(\P^d)}$ and $\mathcal{P}_1(\R^d)$ is a Banach space when equipped with the norm $f \mapsto \norm{\vec{\phi}(f)}_1$. These properties will be important in proving \cref{item:weak*-continuous} of \cref{thm:RBV-top-props}.

\begin{proof}[Proof of \cref{thm:RBV-top-props}, \cref{item:weak*-continuous}]
    Let $(\vec{\phi}, \vec{p})$ be the biorthogonal system constructed in the proof of \cref{thm:RBV-top-props}, \cref{item:RBV-Banach}. Since $\RBV^2(\R^d) \cong \RBV^2_\vec{\phi}(\R^d) \oplus \mathcal{P}_1(\R^d)$,
    showing that $\delta(\dummy - \vec{x}_0)$, $\vec{x}_0 \in \R^d$, is weak$^*$ continuous on $\RBV^2(\R^d)$ is equivalent to showing that it is weak$^*$ continuous on both $\RBV^2_\vec{\phi}(\R^d)$ and $\mathcal{P}_1(\R^d)$.
    
    Clearly $\delta(\dummy - \vec{x}_0)$, $\vec{x}_0 \in \R^d$, is continuous on $\mathcal{P}_1(\R^d)$ (since every element of $\mathcal{P}_1(\R^d)$ is a continuous function). Then, since $\mathcal{P}_1(\R^d)$ is finite-dimensional, the space of continuous linear functionals and weak$^*$ continuous linear functionals are the same. Thus, $\delta(\dummy - \vec{x}_0)$, $\vec{x}_0 \in \R^d$, is weak$^*$ continuous on $\mathcal{P}_1(\R^d)$.

    It remains to show that $\delta(\dummy - \vec{x}_0)$, $\vec{x}_0 \in \R^d$, is weak$^*$ continuous on $\RBV^2_\vec{\phi}(\R^d)$. Let $\X$ be the predual of $\RBV^2_\vec{\phi}(\R^d)$, i.e., $\X' = \RBV^2_\vec{\phi}(\R^d)$. We must show that $\delta(\dummy - \vec{x}_0) \in \X$ for all $\vec{x}_0 \in \R^d$. The Riesz--Markov--Kakutani representation theorem says that the predual of $\M(\P^d)$ is $C_0(\P^d)$. The following diagram shows how all these spaces are related.
    \[
        \begin{tikzcd}[row sep=3em,column sep=5em]
        \RBV^2_\vec{\phi}(\R^d) \arrow[yshift=0.5em]{r}{\ROp} & \M(\P^d) \arrow[yshift=-0.5em]{l}{\ROp^{-1}_\vec{\phi}} \\
        \X  \arrow[dashed]{u}{\text{dual}} \arrow[yshift=-0.5em, swap]{r}{\ROp^{-1*}_\vec{\phi}} & C_0(\P^d) \arrow[dashed, swap]{u}{\text{dual}} \arrow[yshift=0.5em, swap]{l}{\ROp^*_\vec{\phi}}
        \end{tikzcd}
    \]
    
    The above diagram shows that $\delta(\dummy - \vec{x}_0) \in \X$ if and only if $\ROp^{-1*}_\vec{\phi}\curly{\delta(\dummy - \vec{x}_0)} \in C_0(\P^d)$. From \cref{prop:direct-sum-inverse} we see that $\ROp^{-1*}_\vec{\phi}\curly{\delta(\dummy - \vec{x}_0)} = g_\vec{\phi}(\vec{x}_0, \dummy)$ defined in \cref{eq:kernel-of-inverse}. By choosing $\rho = \abs{\dummy} / 2$ in \cref{eq:kernel-of-inverse} we have
    \begin{align*}
        g_\vec{\phi}(\vec{x}_0, (\vec{w}, b))
        &= \frac{\abs*{\vec{w}^\T\vec{x}_0 - b}}{2} - \sum_{k=0}^d p_k(\vec{x}_0) \ang{\phi_k, \frac{\abs*{\vec{w}^\T(\dummy) - b}}{2}} \\
        &\overset{\mathclap{(*)}}{=} \frac{\abs*{\vec{w}^\T\vec{x}_0 - b}}{2} - \sq{\frac{\abs{-b}}{2} + \sum_{k=1}^d x_{0,k} \paren{\frac{\abs*{w_k - b}}{2} - \frac{\abs{-b}}{2}}} \\
        &= \frac{\abs*{\vec{w}^\T\vec{x}_0 - b}}{2} - \frac{\abs{b}}{2}\paren{1 - \sum_{k=1}^d x_{0, k}} - \sum_{k=1}^d x_{0,k} \frac{\abs*{w_k - b}}{2}, \numberthis \label{eq:g-phi}
    \end{align*}
    where $(*)$ follows by substituting in the biorthogonal system $(\vec{\phi}, \vec{p})$ constructed in the proof of \cref{thm:RBV-top-props}, \cref{item:RBV-Banach}. Clearly $g_\vec{\phi}(\vec{x}_0, \dummy)$ is continuous and $g_\vec{\phi}(\vec{x}_0, (\vec{w}, b)) =  g_\vec{\phi}(\vec{x}_0, (-\vec{w}, -b))$, so $g_\vec{\phi}(\vec{x}_0, \dummy)$ is an even function on $\cyl$ and therefore a continuous function on $\P^d$. It remains to check that $g_\vec{\phi}(\vec{x}_0, \dummy)$ is vanishing at infinity. Certainly this is true. Indeed, for sufficiently large $b$ we have
    \[
        g_\vec{\phi}(\vec{x}_0, (\vec{w}, b)) = \frac{-\vec{w}^\T\vec{x}_0 + b}{2} - \frac{b}{2}\paren{1 - \sum_{k=1}^d x_{0, k}} - \sum_{k=1}^d x_{0,k} \frac{-w_k + b}{2} = 0,
    \]
    and for sufficiently small $b$ we have
    \[
        g_\vec{\phi}(\vec{x}_0, (\vec{w}, b)) = \frac{\vec{w}^\T\vec{x}_0 - b}{2} - \frac{-b}{2}\paren{1 - \sum_{k=1}^d x_{0, k}} - \sum_{k=1}^d x_{0,k} \frac{w_k - b}{2} = 0.
    \]
    Therefore, $g_\vec{\phi}(\vec{x}_0, \dummy)$ is compactly supported on $\P^d$ and so $g_\vec{\phi}(\vec{x}_0, \dummy) \in C_0(\P^d)$. Thus, the Dirac impulse $\delta(\dummy - \vec{x}_0)$, $\vec{x}_0 \in \R^d$, is weak$^*$ continuous on $\RBV^2(\R^d)$.
\end{proof}

\section{Proof of \cref{thm:banach-rep-thm}} \label{app:banach-rep-thm}
In order to prove \cref{thm:banach-rep-thm}, we require that solutions to the variational problem in \cref{thm:banach-rep-thm} exist. We will use the following recent result regarding existence of solutions to variational problems over Banach spaces.
\begin{proposition}[{special~case~of~\cite[Theorem~2]{unifying-representer}}] \label{prop:existence}
    Let $(\X, \X')$ be a dual pair of Banach spaces and $\curly{\nu_n}_{n=1}^N \subset \X$ be a collection of linearly independent measurement functionals. Then, the solution set to
    \[
        \argmin_{f \in \X'} \: \norm{f}_{\X'} \quad\subj\quad \ang{\nu_n, f} = y_n, \: n = 1, \ldots, N,
    \]
    is nonempty, convex, and weak$^*$ compact, where $\ang{\dummy, \dummy}$ denotes the pairing of $\X'$ and its continuous dual, $\X''$\footnote{Note that $\nu_n \in \X$ implies $\nu_n \in \X''$ by the canonical embedding of a Banach space in its bidual.}.
\end{proposition}
\begin{remark}
    The result of~\cite[Theorem~2]{unifying-representer} is more general than what is stated in \cref{prop:existence}, but we are only interested in the existence result in this paper.
\end{remark}

\begin{proof}[Proof of \cref{thm:banach-rep-thm}]
    By \cref{thm:RBV-top-props}, we have $\RBV^2(\R^d)$ is a Banach space and that the functionals $\nu_n \coloneqq \delta(\dummy - \vec{x}_n)$, $n = 1, \ldots, N$, are weak$^*$ continuous on $\RBV^2(\R^d)$ (and are therefore contained in the predual of $\RBV^2(\R^d)$). Moreover, this choice of $\curly{\nu_n}_{n=1}^N$ is clearly linearly independent\footnote{Assuming that $\vec{x}_n \neq \vec{x}_k$ for $n \neq k$.}. Therefore, the problem in \cref{eq:banach-rep-variational} satisfies the hypotheses of \cref{prop:existence} and so a solution to \cref{eq:banach-rep-variational} exists. Let $\tilde{s}$ be a (not necessarily unique) solution to \cref{eq:banach-rep-variational}. This solution must be a minimizer of
    \[
        \min_{f \in \RBV^2(\R^d)} \RTV^2(f) \quad\subj\quad 
        \begin{cases}
            f(\vec{x}_n) = y_n, &n = 1, \ldots, N, \\
            f(\vec{0}) = \tilde{s}(\vec{0}), \\
            f(\vec{e}_k) = \tilde{s}(\vec{e}_k), & k = 1, \ldots, d.
        \end{cases}
    \]
    
    By \cref{prop:ridge-spline-rep-thm}, there exists a solution to the above display that takes the form in \cref{eq:banach-rep-soln} with $K \leq N$ neurons, so we can always find a solution to the original problem in \cref{eq:banach-rep-variational} of the form in \cref{eq:banach-rep-soln}.
\end{proof}

\section{Proof of \cref{thm:vv-banach-rep-thm}} \label{app:vv-banach-rep-thm}
\begin{proof}
    By \cref{lemma:vv-top-props}, we have that $\RBV^2(\R^d; \R^D)$ is a Banach space and that the point evaluation operator is component-wise weak$^*$ continuous on $\RBV^2(\R^d; \R^d)$. Therefore, the functionals 
    \[
        \ang{\nu_{n, m}, f} = f_m(\vec{x}_n), \: n = 1, \ldots, N, \: m = 1, \ldots, D,
    \]
    where $f = (f_1, \ldots, f_D) \in \RBV^2(\R^d; \R^D)$ and $\ang{\dummy, \dummy}$ denotes the pairing of $\RBV^2(\R^d; \R^D)$ and its continuous dual, are contained in the predual of $\RBV^2(\R^d; \R^D)$. Moreover, these functionals are linearly independent\footnote{Assuming that $\vec{x}_n \neq \vec{x}_k$ for $n \neq k$.}. Therefore, the problem in \cref{eq:vv-banach-rep-variational} satisfies the hypotheses of \cref{prop:existence} and so a solution to \cref{eq:vv-banach-rep-variational} exists. Next, note that we
    can rewrite the problem in \cref{eq:vv-banach-rep-variational} as
    \[
        \min_{\substack{f = (f_1, \ldots, f_D) \\ f_m \in \RBV^2(\R^d) \\ m = 1, \ldots, D}} \: \sum_{m=1}^D \norm{f_m}_{\RBV^2(\R^d)} \quad\subj\quad f_m(\vec{x}_n) = y_{n, m}, \: \begin{cases}
        n = 1, \ldots, N \\
        m = 1, \ldots, D,
        \end{cases}
    \]
    where $\vec{y}_n = (y_{n, 1}, \ldots, y_{n, D}) \in \R^D$. Let $\tilde{s} = (\tilde{s}_1, \ldots, \tilde{s}_D)$ be a (not necessarily unique) solution to \cref{eq:vv-banach-rep-variational}. From the above display we see that this solution must satisfy
    \begin{equation}
        \tilde{s}_m \in \argmin_{f \in \RBV^2(\R^d)} \norm{f}_{\RBV^2(\R^d)} \quad\subj\quad f(\vec{x}_n) = y_{n, m}, \: n=1, \ldots, N,
        \label{eq:sm-minimizer}
    \end{equation}
    for $m = 1, \ldots, D$. To see this, note that if the above display did not hold, it would contradict the optimality of $\tilde{s}$. By \cref{thm:banach-rep-thm}, there exists a solution to the above display that takes the form in \cref{eq:banach-rep-soln} with $K_m \leq N$ neurons. By combining these solutions into a single vector-valued function with potential combining of neurons\footnote{This would happen in the event that $\tilde{s}_m$ and $\tilde{s}_\ell$, $m \neq \ell$, shared a common neuron.} we see that there exists a solution to the original problem in \cref{eq:vv-banach-rep-variational} that takes the form in \cref{eq:vv-banach-rep-soln} with $K \leq K_1 + \cdots + K_D \leq ND$ neurons. If no neurons combine, each $\vec{v}_k$ is $1$-sparse.
\end{proof}
\begin{remark}\label{rem:no-sharing}
    One could also write a solution of \cref{eq:vv-banach-rep-variational} such that each output is completely independent of any other output, i.e., the outputs are completely decoupled. This corresponds to fitting the data with $D$ separate single-hidden layer ReLU networks. This follows from the fact that $s_m$ is a minimizer to the problem in \cref{eq:sm-minimizer}. This corresponds to the representation in \cref{eq:vv-banach-rep-soln} having each $\vec{v}_k$ being $1$-sparse.
\end{remark}

\section{Proof of \cref{lemma:Lipschitz-bound}} \label{app:Lipschitz-bound}
Before proving \cref{lemma:Lipschitz-bound}, we will first bound the Lipschitz constant of functions in $\RBV^2(\R^d)$. To do this, we will rely on \cref{prop:direct-sum-inverse} with the biorthogonal system constructed in the proof of \cref{thm:RBV-top-props} given in \cref{app:RBV-top-props}. In particular, \cref{prop:direct-sum-inverse} provides the direct-sum decomposition of $f \in \RBV^2(\R^d)$ by
\begin{equation}
    f(\vec{x}) = \int_\cyl g_\vec{\phi}(\vec{x}, (\vec{w}, b)) u(\vec{w}, b) \dd \sigma(\vec{w}) \dd b + \vec{c}^\T\vec{x} + c_0,
    \label{eq:integral-representation}
\end{equation}
with $g_\vec{\phi}$ as in \cref{eq:g-phi}. It can easily be checked that this decomposition has the property that
\begin{equation}
    \norm{f}_{\RBV^2(\R^d)} = \norm{u}_{\M(\cyl)} + \norm{\vec{c}}_1 + \abs{c_0},
    \label{eq:direct-sum-norm}
\end{equation}
and we refer the reader to~\cite[Theorem~22, Item~3]{ridge-splines} for more details.
\begin{lemma} \label{lemma:Lipschitz-bound-scalar}
    Let $f \in \RBV^2(\R^d)$. Then, $f$ is Lipschitz continuous and satisfies the Lipschitz bound
    \[
        \abs{f(\vec{x}) - f(\vec{y})} \leq  \norm{f}_{\RBV^2(\R^d)} \, \norm{\vec{x} - \vec{y}}_1.
    \]
\end{lemma}
\begin{proof}
    We will first bound the Lipschitz constant of $g_\vec{\phi}(\dummy, \vec{z})$ defined in \cref{eq:g-phi}, where $\vec{z} = (\vec{w}, b) \in \cyl$. For any $\vec{x}, \vec{y} \in \R^d$,
    \begin{align*}
        \abs{g_\vec{\phi}(\vec{x}, \vec{z}) - g_\vec{\phi}(\vec{y}, \vec{z})}
        &= \abs*[\Bigg]{\frac{\abs{\vec{w}^\T\vec{x} - b}}{2} - \frac{\abs{\vec{w}^\T\vec{y} - b}}{2} \\
        &\qquad - \frac{\abs{b}}{2} \sq{\paren{1 - \sum_{k=1}^d x_k} - \paren{1 - \sum_{k=1}^d y_k}}
        - \sum_{k=1}^d (x_k - y_k) \frac{\abs{w_k - b}}{2}} \\
        &\leq \frac{\abs{\,\abs{\vec{w}^\T\vec{x} - b} - \abs{\vec{w}^\T\vec{y} - b}\,}}{2} \\
        &\qquad + \abs{\sum_{k=1}^d (x_k - y_k) \frac{\abs{b}}{2} - \sum_{k=1}^d (x_k - y_k) \frac{\abs{w_k - b}}{2}} \\
        &\leq \frac{\abs{\,\abs{\vec{w}^\T\vec{x} - b} - \abs{\vec{w}^\T\vec{y} - b}\,}}{2}
        + \sum_{k=1}^d \abs{x_k - y_k} \frac{\abs{\,\abs{b} - \abs{w_k - b}\,}}{2} \\
        &\overset{\mathclap{(*)}}{\leq} \frac{\abs{\vec{w}^\T\vec{x} - \vec{w}^\T\vec{y}}}{2} + \sum_{k=1}^d \abs{x_k - y_k} \frac{\abs{w_k}}{2} \\
        &\overset{\mathclap{(\mathsection)}}{\leq} \frac{\norm{\vec{w}}_\infty \norm{\vec{x} - \vec{y}}_1 + \norm{\vec{w}}_\infty\norm{\vec{x} - \vec{y}}_1}{2} \\
        &\overset{\mathclap{(\dagger)}}{\leq} \norm{\vec{x} - \vec{y}}_1
    \end{align*}
    where $(*)$ holds from the reverse triangle inequality, $(\mathsection)$ holds from H\"older's inequality, and $(\dagger)$ holds from the fact that $\norm{\dummy}_\infty \leq \norm{\dummy}_2$ in finite-dimensional spaces combined with $\norm{\vec{w}}_2 = 1$.
    
    Next, from \cref{eq:integral-representation} we have for any $\vec{x}, \vec{y} \in \R^d$,
    \begin{align*}
        \abs{f(\vec{x}) - f(\vec{y})}
        &\leq \int_\cyl \abs{g(\vec{x}, (\vec{w}, b)) - g(\vec{y}, (\vec{w}, b))} \abs{u(\vec{w}, b)} \dd\sigma(\vec{w}) \dd b + \abs*{\vec{c}^\T(\vec{x} - \vec{y})} \\
        &\leq \int_\cyl \norm{\vec{x} - \vec{y}}_1 \abs{u(\vec{w}, b)} \dd\sigma(\vec{w}) \dd b + \norm{\vec{c}}_\infty \norm{\vec{x} - \vec{y}}_1 \\
        &\leq \norm{u}_{\M(\cyl)}\norm{\vec{x} - \vec{y}}_1 + \norm{\vec{c}}_1 \norm{\vec{x} - \vec{y}}_1 \\
        &\leq \norm{f}_{\RBV^2(\R^d)} \, \norm{\vec{x} - \vec{y}}_1,
    \end{align*}
    where the third line follows from the fact that $\norm{\dummy}_\infty \leq \norm{\dummy}_1$ in finite-dimensional spaces and the fourth line follows from \cref{eq:direct-sum-norm}.
    
\end{proof}

\begin{proof}[Proof of \cref{lemma:Lipschitz-bound}]
    Write $f = (f_1, \ldots, f_D)$. For any $\vec{x}, \vec{y} \in \R^d$,
    \begin{align*}
        \norm{f(\vec{x}) - f(\vec{y})}_1
        &= \sum_{m=1}^D \abs{f_m(\vec{x}) - f_m(\vec{y})} \\
        &\leq \paren{\sum_{m=1}^D \norm{f_m}_{\RBV^2(\R^d)}} \norm{\vec{x} - \vec{y}}_1, \\
        &= \norm{f}_{\RBV^2(\R^d; \R^D)} \norm{\vec{x} - \vec{y}}_1,
    \end{align*}
    where the second line follows from \cref{lemma:Lipschitz-bound-scalar} and the third line follows from the definition of $\norm{\dummy}_{\RBV^2(\R^d; \R^D)}$ in \cref{lemma:vv-top-props}.
\end{proof}

\begin{remark}
    The Lipschitz bounds in \cref{lemma:Lipschitz-bound-scalar,lemma:Lipschitz-bound} are by no means the tightest Lipschitz bounds.
\end{remark}


\bibliographystyle{siamplain}
\bibliography{ref}

\end{document}